\documentclass{article}
\usepackage[accepted]{aistats2022}

\usepackage[round]{natbib}

\usepackage[utf8]{inputenc} % allow utf-8 input
\usepackage[T1]{fontenc}    % use 8-bit T1 fonts
\usepackage{hyperref}       % hyperlinks
\usepackage{url}            % simple URL typesetting
\usepackage{booktabs}       % professional-quality tables
\usepackage{amsfonts}       % blackboard math symbols
\usepackage{nicefrac}       % compact symbols for 1/2, etc.
\usepackage{microtype}      % microtypography
\usepackage{xcolor}         % colors
\usepackage{graphicx}
\usepackage{wrapfig}
\usepackage{subcaption}

\usepackage{commands}
\usepackage{amsmath}
\usepackage{amsthm,amssymb, mathtools}

\newcommand{\fedssgd}{\texttt{FetchSGD}}

\newcommand{\fedavg}{\texttt{FedAvg}}
\newcommand{\algoname}{\texttt{SparseFed}}

\newcommand{\comment}[1]{}

\begin{document}
\onecolumn
\aistatstitle{SparseFed: Mitigating Model Poisoning Attacks in Federated Learning with Sparsification}

\aistatsauthor{Ashwinee Panda, Saeed Mahloujifar, Arjun N. Bhagoji, Supriyo Chakraborty, Prateek Mittal}
% \aistatsauthor{Ashwinee Panda \and }
% \aistatsaddress{Princeton University \and }
\aistatsaddress{Princeton University, University of Chicago, IBM}

\begin{abstract}
% Existing approaches to federated learning are vulnerable to targeted model poisoning attacks, wherein an adversary present during training uses poisoned parameter updates to insert backdoors into the global model. 
Federated learning is inherently vulnerable to model poisoning attacks because its decentralized nature allows attackers to participate with compromised devices.
In model poisoning attacks, the attacker reduces the model's performance on targeted sub-tasks (e.g. classifying planes as birds) by uploading "poisoned" updates.
% Recent works have examined the vulnerability of federated learning systems to backdoor attacks, but as we show, proposed defenses do not perform well in practice and their provable guarantees are insufficient.
In this report we introduce \algoname{}, a novel defense that uses global top-k update sparsification and device-level gradient clipping to mitigate model poisoning attacks.
We propose a theoretical framework for analyzing the robustness of defenses against poisoning attacks, and provide robustness and convergence analysis of our algorithm.
To validate its empirical efficacy we conduct an open-source evaluation at scale across multiple benchmark datasets for computer vision and federated learning.
% By uploading adversarial updates (\emph{model poisoning})
% , even a small fraction of attackers (malicious devices) can cause the model to misclassify many datapoints.
% A key insight in the design of \algoname{} is that top-k update sparsification only updates model parameters in the directions that most clients are updating, which makes it difficult for a small number of attackers to backdoor the model.
% We show that while backdoor attacks succeed against existing defenses, \algoname{} reduces attack accuracy significantly and does not compromise test accuracy in the absence of an attack.
% and design attacks which obtain best-in-class performance against existing defenses, which our defense nevertheless thwarts.
\end{abstract}
\vspace*{-13pt}
\section{INTRODUCTION}\label{sec: intro}
\vspace*{-2pt}
% Federated learning is a paradigm for distributed machine learning that is being adopted and deployed at scale by large corporations
% \citep{mcmahan17fedavg, kairouz2019advances}.
% In the federated paradigm, training data is distributed across a large number of consumer devices. 
% Aggregating data from these devices poses privacy, legal and computational challenges \citep{kairouz2019advances}. 
% In federated learning, clients do not upload their local data to a central aggregator, and instead only communicate locally computed model updates. 
% In this work we address the challenge of training high-quality models in the presence of attackers on a federated system. 

% \begin{table}
%   \caption{SparseFed is the only defense that can significantly reduce the attack accuracy against colluding attackers}
%   \label{table:intro}
%   \centering
%   \begin{tabular}{lllll}
%     \toprule
%     \multicolumn{2}{c}{Attack Accuracy (Dataset)}                   \\
%     \cmidrule(r){1-2}
%     Name     & CIFAR10 & CIFAR100 & FMNIST & FEMNIST \\
%     \midrule
%     \bottomrule
%   \end{tabular}
% \end{table}
The federated learning paradigm enables training models across consumer devices without aggregating data, but deployed systems are not robust to model poisoning attacks~\citep{wang2020attack, pmlr-v97-bhagoji19a, bagdasaryan18backdoor}.
There are two main settings for federated learning: the cross-device setting and the cross-silo setting~\citep{kairouz2019advances}.
% In the cross-silo setting, typically $\leq 1000$ parties in training a model \citep{kairouz2019advances}.
In the cross-device setting, the goal is to train a model across disjoint data distributed across many thousands of devices~\citep{kairouz2019advances}.
In the cross-silo setting, data distributions are less extreme and fewer devices participate~\citep{kairouz2019advances}.
Compromised devices are easily able to participate in federated learning~\citep{bonawitz19sysml} and the models trained are often redeployed to serve millions or billions of requests~\citep{hard18gboard}.
Attackers often have an incentive to compromise the behavior of trained models~\citep{pmlr-v97-bhagoji19a, bagdasaryan18backdoor}.
In this work we focus on targeted model poisoning attacks, wherein the attackers' goal is to reduce the model's performance on a specific set of datapoints from the test distribution or on certain sub-tasks using corrupted model updates, without compromising test accuracy.

% \begin{figure*}
%     \centering
%     \vspace{-0.5cm}
%     \includegraphics[height=5cm]{images/fedsketch_fig.pdf}
%     \vspace{-1.25cm}
%     \caption{\textbf{Algorithm Overview}.  
%     The \algoname{}{} algorithm \textbf{(1)} computes gradients locally, with benign devices computing the gradient of their local dataset and attackers computing the gradient over their auxiliary dataset and performing PGD, and then \textbf{(2)} the gradients are clipped (enforced by the server).  In the cloud, updates are aggregated \textbf{(3)}, and the top-$k$ values are then \textbf{(4)} extracted and \textbf{(5)} broadcast as sparse updates to devices participating in the next round. The clipping and top-$k$ extraction serve to mitigate the impact of the malicious update (red matrix).
%     (Figure adapted from Fig. 1 of  \citep{rothchild2020fetchsgd})}
%     \vspace{-0.5cm}
%     \label{fig:overview}
% \end{figure*}

% \begin{wrapfigure}{R}{0.4\textwidth}
\begin{figure}
    \centering
    \includegraphics[width=0.48\textwidth]{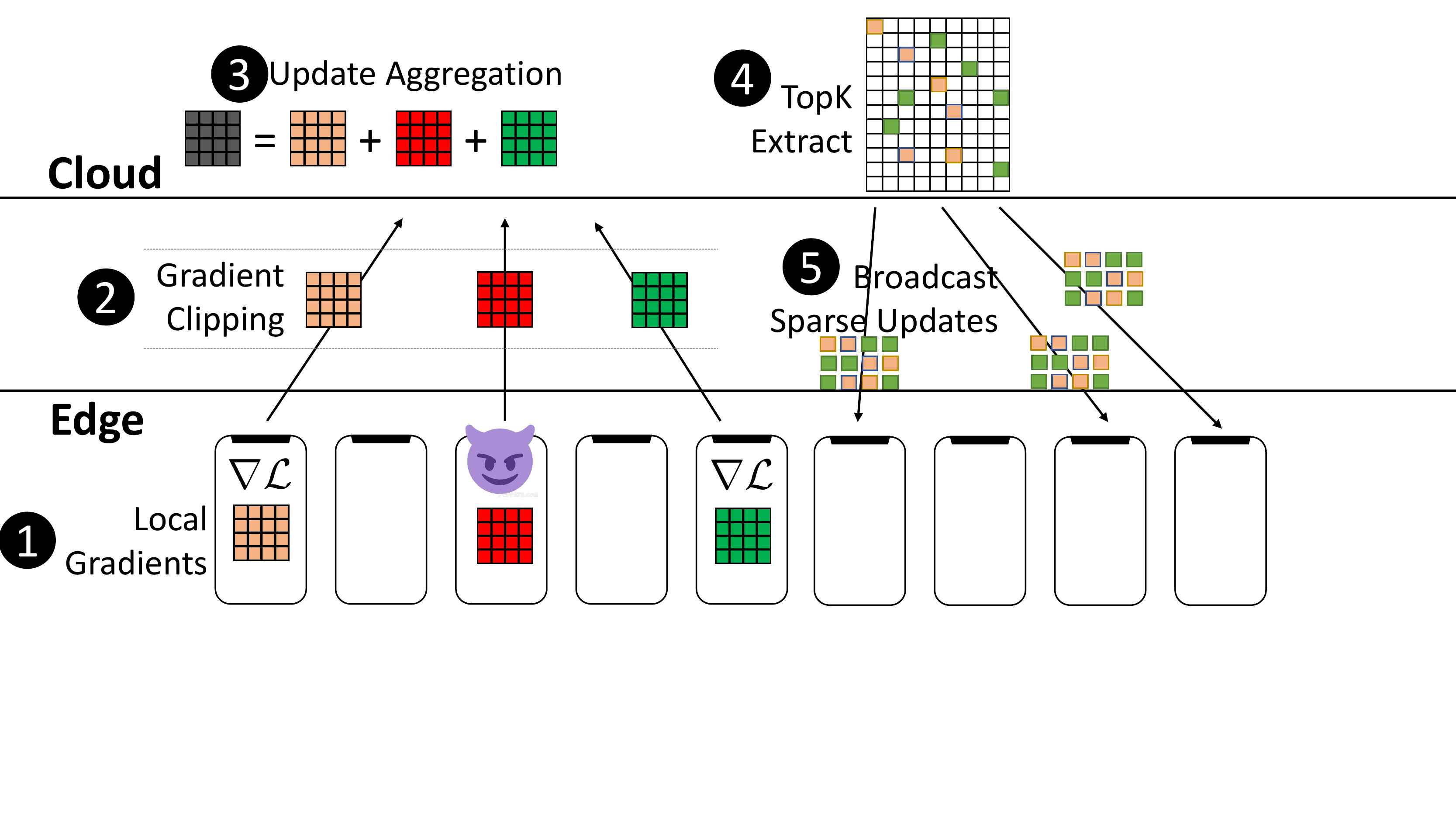}
    \setlength{\belowcaptionskip}{-10pt}
    \setlength{\abovecaptionskip}{-10pt}
        \caption{\textbf{Algorithm Overview}.  
    The \algoname{}{} algorithm \textbf{(1)} computes gradients locally, 
    % with benign devices computing the gradient of their local dataset and attackers computing the gradient over their auxiliary dataset and performing PGD, 
    and then \textbf{(2)} the gradients are clipped.  
    In the cloud, updates are aggregated \textbf{(3)}, and the $top_k$ values are then \textbf{(4)} extracted and \textbf{(5)} broadcast as sparse updates to devices participating in the next round.
    The clipping and $top_k$ extraction serve to mitigate the impact of the malicious update (red matrix).
    % (Figure adapted from Fig. 1 of  \citep{rothchild2020fetchsgd})
    }
    \label{fig:algo}
\end{figure}
%   \vspace*{-pt}
% \end{wrapfigure}

The constraints of operating in the cross-device federated setting present challenges that make it difficult to train a model without enabling attackers.
The data available across devices is not independent and identically distributed (non-i.i.d.).
For example: when training a classification model on the camera roll of smartphone users, devices belonging to cat and dog owners will generate data from different distributions, but we are still interested in training one model to distinguish between cats and dogs~\citep{hard18gboard}.
Therefore many benign device gradients will be very far apart in $\ell_2$ distance, so heuristics that eliminate gradients that are outliers may not function well~\citep{zhao2018federated, Rousseeuw1985}.
% Existing federated learning deployments make use of secure aggregation~\citep{bonawitz17secagg} to preserve privacy of device data.
% This protocol prevents the server from inspecting individual gradients, which makes it difficult to implement most existing defenses that rely on computing statistical properties over gradient updates~\citep{mhamdi2018hidden, blanchard2017machine, chen18draco}.
Devices only participate once during all of training~\citep{kairouz2019advances}, and this makes it difficult to use historical reputation mechanisms to shut out attackers~\citep{cao2020fltrust}.

\textbf{Contributions.} In this work, we present \algoname{}, a new optimization algorithm for federated learning that can \emph{train high-quality models under these constraints while greatly mitigating model poisoning attacks}.
We describe \algoname{} in detail in Section \ref{sec: defense}, but the main idea is intuitive: at each round, participating devices compute an update on their local data and clip the update.
% (verified by the server, Appendix B.9)
% \prateek{this reference looks strange...specially in the intro + since in this paper we are not talking about secure aggregation anymore}
% \supriyo{Do we need to provide a sketch (maybe in Appendix with forward pointer here) of the validation procedure at the server? In the rebuttal you had defined this for both secure (zk-snark-based) and non-secure aggregation})
The server computes the aggregate gradient, and only updates the $top_k$ highest magnitude elements.
Because attackers will necessarily be moving in distinct directions from the majority of benign devices, the coordinates the attackers need to update in order to poison the model usually will not be updated.
% \saeed{Do we want to keep this intuition?} 
% new stuff
Our protocol is a defense at training time, and is complementary to the line of work that proposes test-time modifications for robustness such as smoothing~\citep{xie2021crfl, weber2021rab}.
Prior defenses at training time use Byzantine-robust learning algorithms that bound the single iteration deviation between poisoned and clean models~\citep{mhamdi2018hidden,blanchard2017machine}.
However, the iterative nature of learning ensures that small deviations at the start of training compound exponentially.
% Therefore, Byzantine-robust defenses cannot achieve \emph{certified robustness.}

We propose a framework for analyzing the robustness of defenses under the \emph{certified radius} metric from prior work~\citep{xie2021crfl}.
The certified radius is an upper bound on the distance that a poisoned model can drift from a benign model, and limits the impact that an attacker can have on the model.
Under our framework, \algoname{} minimizes the certified radius by sparsifying the aggregate model updates. 

% A learning algorithm is certifiably robust against model poisoning attacks if, given a number of compromised devices, the algorithm can learn a classifier that predicts the same label for a particular testing example even when attackers are present.
% We provide a procedure for certifying the robustness of \algoname{}, and as per Table \ref{table:cert} we can certify more datapoints than prior work.

We validate the effectiveness of our method empirically on four benchmark computer vision datasets and one natural language processing dataset, training models with between 6 and 40 million parameters on non-i.i.d.\ datasets that range between 50,000 and 800,000 examples.
We evaluate \algoname{} against four attacks from prior work \citep{pmlr-v97-bhagoji19a, bagdasaryan18backdoor, Fang2020LocalMP, sun2019backdoor} and two new attacks we introduce, in the cross-silo and cross-device settings.
As we show in Table \ref{table:eval}, \algoname{} does not degrade test accuracy by more than $1 \%$, mitigates attack accuracy, e.g. by over $97 \%$ on the FEMNIST dataset, and significantly outperforms prior work.
% As shown in Table \ref{table:eval}, our algorithm reduces the attack accuracy significantly more than the next best defense does, and does not compromise test accuracy in the absence of an attack.
The code to implement our defense is \textcolor{blue}{\href{https://github.com/sparsefed/sparsefed}{open-source}}.
\vspace*{-10pt}
\section{SPARSEFED}\label{sec: defense}
\vspace*{-10pt}
In this section we introduce a framework for analyzing the robustness of machine learning protocols against poisoning attacks.
We use this framework to motivate \algoname{}, that uses gradient sparsification to mitigate attackers, and provide a theoretical analysis of its robustness, convergence and efficiency.
The key tool we use is the \emph{certified radius}, that is the upper bound on the distance between poisoned and benign models. 

\subsection{Certified radius as a framework for robustness}

\noindent \textbf{Notation:}
Let $Z$ be the data domain and $D^t$ be data sampled (not necessarily i.i.d.) from $Z$ at iteration $t$.
Let $\Theta$ be the class of models in $d$ dimensions, and $\cL: \Theta \times Z^* \rightarrow \cR$ be a loss function. 
% \supriyo{Is $Z^*$ for a possibly empty domain?}
A protocol $f=(\cG,\cA, \lambda)$ consists of a gradient oracle 
$\cG(\theta,D,t) \rightarrow \cR^d$ that takes a model, a dataset and a round index and outputs the update vector $u^t$. $f$ also includes an update algorithm $\cA:u^t \in \cR^d \rightarrow \cR^d$, e.g. momentum.
% \supriyo{Should it be $\lambda(t) \in \cR$} 
$\lambda(t) \in \cR$ is a learning rate scheduler, possibly static, and $\Lambda(t)$ the cumulative learning rate $\Lambda(t)=\sum_{i=1}^t\lambda(t)$. 
The update rule of the protocol is then defined as $\theta_{t+1} = \theta_{t} - \lambda(t) \cA(u^t)$.

\begin{definition}[Poisoning Attack]\label{def:attack}
For a protocol $f=(\cG,\cA, \lambda)$ we define the set of \textbf{poisoned} protocols $F(\rho)$ to be all protocols $f^*=(\cG^*,\cA, \lambda)$ that are exactly the same as $f$ except that the gradient oracle $\cG^*$ is a $\rho$-corrupted version of $\cG$. 
That is, for any round $t$ and any model $\theta_t$ and any dataset $D$ we have we have $\cG^*(\theta_t, D)= \cG(\theta_t, D) +\epsilon$ for some $\epsilon$ with $||\epsilon||_1 \leq \rho$.
\end{definition}

\begin{remark}
Under our attack model, the attacker can contribute to the update with a vector $\epsilon$ of $\ell_2$ mass at most $\rho$.
This model generalizes existing defenses, e.g. $\ell_2$ clipping and Byzantine resilient aggregation rules \citep{bulyan}.
\end{remark}

\begin{definition}[Certified Radius]\label{def:radius}
Let $f$ be a protocol and $f^*\in F(\rho)$ be the a poisoned version of the same protocol.
Let $\theta_T, \theta^{*}_T$ be the benign and poisoned final outputs of the above protocols.
We call $R$ a certified radius for $f$ if 
$\forall f^*\in F(\rho);  R(\rho) \geq   |\theta_T - \theta^{*}_T|_1.$
\end{definition}
\noindent \textbf{Robustness Against Poisoning}
The \emph{certified radius} has been established as a metric of the strength of defenses~\citep{xie2021crfl}.
Prior work has analyzed the certified radius in two ways.
The first is minimizing the divergence between the benign and poisoned protocols in a single iteration, as in~\citep{blanchard2017machine, bulyan, krum, xie2021crfl}.
As per \citep{xie2021crfl}, we know that a small certified radius improves robustness because models that are very close to each other are likely to predict the same label for the same datapoint.
However, these papers assume i.i.d. data~\citep{bulyan, krum, blanchard2017machine} and do not consider the \emph{propagation error}: that small changes in early iterations can quickly compound and create a large divergence in the model.
Therefore, defenses that aim to minimize the divergence in a single iteration via outlier detection or any other strategy cannot provide guarantees in the cross-device setting.
The second is combinatorial bounds via ensembling \citep{jia2020intrinsic, cao2021provably}.
Combinatorial bounds do not compute the certified radius, and instead directly bound the change in the label probabilities.
However, combinatorial bounds do not scale to the cross-device setting.
For instance, the guarantees of \citep{cao2021provably} only hold so long as $\binom{n}{k}<2 \binom{n - m}{k}$ where $n$ is the number of devices, $m$ is the number of compromised devices, and $k$ is the size of the ensemble (equation 4 in \citep{cao2021provably}) which is generally $1\%$ of $n$.
For $n \gtrsim 10^4$ (the cross-device setting), this means that \citep{cao2021provably} and other ensembling strategies cannot provide any guarantees when $m > 0.5 \%$ of n.

\emph{In this section we introduce a framework for analyzing the certified radius of poisoning attacks in the cross-device setting.}

% \prateek{seems like a strong claim given prior works on certified robustness. not sure about this}

\noindent \textbf{Analyzing Propagation Error}
We conduct $T$ rounds of the protocol $f$: at round $i\in[T]$ we receive an update, and use the output of the update algorithm $\cA(u^t)$ to compute the new model $\theta_{t+1}$.
At each iteration, the upper bound $\rho$ on $\epsilon$ gives the \emph{additive error} introduced by poisoning.
Because the protocol is adaptive, small additive errors introduced at early iterations can build upon each other and create large divergence.
We refer to this as the \emph{propagation error}.
To analyze the propagation error we use the protocol Lipschitzness, defined in Definition \ref{def:clipschitz}.

% \begin{definition}[Update Lipschitz]\label{def:lipschitz}
% A protocol $f:D \rightarrow \cR^n$ is L-Lipschitz if $\forall x,y \in D |f(x)-f(y)|_1 \leq L|x-y|_1$
% \end{definition}

\begin{definition}[Coordinate Lipschitz]\label{def:clipschitz}
A protocol $f(\cG,\cA, \lambda)$ is $c$-coordinatewise Lipschitz if for any round $t\in[T]$, models $\theta_t, \theta^{*}_t \in \cM$, and a dataset $D$ we have that the outputs of the gradient oracle on any coordinate cannot drift too much farther apart. Specifically, for any coordinate index $i\in[d]$
$$\Big|\cG(\theta^*_t, D)[i] - \cG(\theta_t,D)[i]\Big|\leq c\cdot |\theta^{*}_t-\theta_t|_1.$$
\end{definition}

\begin{example}[Training a single layer neural network with SGD]\label{example:singlelayer}
% In this example, we show that the coordinatewise Lipschitz constant of the SGD protocol for a single layer neural network is $c=\frac{1}{4}$.
% We provide the full computation in Appendix B.1.
% \begin{equation}
%     \sup_{x, \theta_1, \theta_2} | g(x, \theta_1)_{i^{*} j^{*}} - g(x, \theta_1)_{i^{*} j^{*}} | \leq \frac{1}{4} |\theta_1 - \theta_2|
% \end{equation}
In this example, we compute the coordinatewise Lipschitz constant of the SGD protocol for a single layer neural network defined as $\sigma(\theta x)$, where $\sigma$ is the softmax function and $\theta \in \cR^d$ are the network parameters. For cross-entropy loss-based training using dataset $D$, we show that the constant $c=\frac{1}{4}$. Formally,
\begin{equation}
\sup_{D, \theta_1, \theta_2} | g(D, \theta_1)[ i ]  - g(D, \theta_2)[ i ] |_1 \leq \frac{1}{4} | \theta_1 - \theta_2 |_1 \;\; \forall i \in [d] \nonumber
\end{equation}
where $g(D, \theta)[i] = \frac{\partial \mathcal{L}}{\partial \theta_i}$. We provide the full computation in Appendix \ref{appendix:singlelayer}.
\end{example}

% \noindent \textbf{Lipschitzness of Deep Networks}
% The Lipschitz constant is an important theoretical quantity, and recent work in the neural tangent kernel (NTK) regime has made progress in computing the Lipschitz constant of deep ReLU networks, the functions we are interested in analyzing~\citep{jordan2021exactly, liu2021linearity, liu2021loss}.
% However, because it is still difficult to compute the Lipschitz constant of general networks, we simplify the task by introducing the intermediate coordinatewise Lipschitz constant, that we define in Definition \ref{def:clipschitz}.

% As we show in Example \ref{example:singlelayer}, the coordinatewise Lipschitz constant is straightforward to compute, and serves to simplify our analysis of the overall protocol Lipschitzness in Lemma \ref{lem:lipschitz}.

% \begin{lemma}\label{lem:lipschitz}
% A $c$-coordinatewise Lipschitz protocol $f$ that updates $d$ coordinates is at most $c \cdot d$-Lipschitz.
% \end{lemma}
\noindent \textbf{Analyzing the Certified Radius}
In Theorem \ref{theorem:framework}, we account for the propagation error and obtain a certified radius for general protocols. 
We provide a procedure for computing the certified radius exactly in Appendix \ref{appendix:computingradius}.
% This framework for computing the certified radius in the cross-device setting is \emph{general}.
Unlike prior work, we do not make any assumptions on the distribution of data across devices~\citep{bulyan}, the number of iterations where the attacker is present \citep{xie2021crfl}, the number of devices \citep{cao2021provably}, or the number of poisoned points \citep{jia2020intrinsic}.
% or the determinism of training \saeed{this last one is a bit hard to explain. Certified radius is by definition for deterministic training, but you can extend it to random as well.} \citep{rosenfeld2020certified}.
We can account for these factors by adjusting the relevant quantities.
Although the computed certified radius from Theorem \ref{theorem:framework} may not be tight, we expect protocols that improve the bound to benefit from improvements in their robustness. 
In the next section, we show one way to improve this bound with sparsification by decreasing the propagation error.
% \saeed{Slightly changed the last sentence...}.

% \begin{theorem}[General Certified Radius]\label{theorem:framework}
% Let $f$ be a $c$-coordinatewise-Lipschitz protocol that updates a model with $d$ weights, then
% $R(\rho) = \Lambda(T) (1+d c)^{\Lambda(T)}\rho$
% is a certified radius for $f$.
% \end{theorem}

\begin{theorem}\label{theorem:framework}
Let $f$ be a $c$-coordinatewise-Lipschitz protocol on a dataset $D$. Then
$R(\rho) = \Lambda(T) (1+dc)^{\Lambda(T)}\rho$
is a certified radius for $f$.
\end{theorem}
% \saeed{Theorem \ref{theorem:framework} shows how one can account for the "propagation" error and get a certified radius. In what follows, we show how this bound improves with sparsity. In particular, we show that a sparse update rule that approximate the original rule can achieve less propagation error.}
\subsection{Security analysis of \algoname{}}
In this section we use our certified radius framework to motivate \algoname{}, that uses gradient sparsification and norm clipping to mitigate attackers, and provide a theoretical analysis of its robustness.

\noindent \textbf{The building blocks of robustness}
The two components of the certified radius are the additive error and the propagation error.
The additive error represents the attacker's power in terms of an upper bound $\rho$ on the noise vector $\epsilon$.
We can enforce this with device level $\ell_2$ gradient norm clipping, that is a standard technique employed by prior work \citep{sun2019backdoor,wang2020attack}.
If $p \%$ of devices are compromised and the parameter of $\ell_2$ clipping is $L$ then $\rho = pL$.
The propagation error represents the protocol's inherent robustness in terms of the Lipschitz constant $c \cdot d$.
% It is straightforward to minimize this term by updating fewer coordinates, but if this in turn increases the number of iterations required for convergence, the overall certified radius can increase
% \saeed{What is the purpose of this sentence?} \saeed{alternate sentence: The propagation error, accounts for the role of benign parties in helping adversary to achieve it's goal. At each iteration, the benign parties calculate their gradient on the poisoned model. This gradient will be different from the gradient calculated on the benign model and hence, this different must be accounted for in certified radius calculation. We call this the propagation error. }.

Update sparsification techniques reduce the number of non-zero entries in the aggregated stochastic gradient before it is applied to the global model.
Global $top_k$ sparsification \citep{Stich2018SparsifiedSW} is one such method that updates only the $k$ coordinates with the largest magnitude, where $k \|| d$, and converges at the same rate as SGD \citep{karimireddy2019ef}.
To the best of our knowledge, we are the first to propose the use of global update sparsification as a building block for robust federated learning.

\begin{algorithm}
\caption{\algoname{}}
\label{alg:sparsefed}
    \begin{algorithmic}
    {\begin{small}
	\REQUIRE number of coordinates to update each round $k$, learning rate $\lambda$, number of timesteps $T$, local batch size $b$, number of devices selected per round $n$, norm clipping parameter $L$, local epochs $\tau$, local learning rate $\gamma$, device datasets $D_{j=1}^n$, momentum $\rho$
    \STATE Initialize model $\theta_0$ using the same random seed on the devices and aggregator 
    \STATE Initialize memory vector $\W_t=0$
    , momentum vector $\R^t=0$
	\FOR{$t = 1,2,\cdots T$}
	    \STATE Randomly select $n$ devices $d_1,\ldots d_n$
	    \LOOP[In parallel on devices $\bc{d_i}_{i=1}^n$]
	        \STATE Download new model weights $\theta_t = \theta$ 
	        \FOR{$m \in \tau$}
        		\STATE Compute gradient $\g_t^i = \frac{1}{b}\sum_{j=1}^l\nabla_\theta \mathcal{L}(\theta^t, \D_j)$
        		\STATE Accumulate gradient $\theta_t = \theta_t - \gamma(t, m) \g_t^i$ 
    		\ENDFOR
    		\STATE Compute update $\u_t^i = \theta_t - \theta$
    		\STATE Clip update $\u_t^i = \u_t^i \cdot \min(1, \frac{L}{|\u_t^i|_2})$
    % 		(can be verified by MPC, e.g. range proof)
    	\ENDLOOP
		\STATE Aggregate gradients $\u_t = \frac{1}{n}\sum_{i=1}^n \u_t^i$ 
		\STATE Momentum: $\R^t=\rho \R^{t-1}+u^t$
 		\STATE Error feedback: $\W_t = \u_t + \W_t$
 		\STATE Extract $top_k$: $\Delta_t = top_k(\W_t)$
		\STATE Error accumulation: $\W_{t+1} = \W_t -\Delta_t$
		\STATE Momentum factor masking: $\R_{t+1} = \R_t - \Delta_t$
	    \STATE Update $ \theta_{t+1} = \theta_{t} - \lambda(t) \Delta_t$ 
    \ENDFOR 
	\ENSURE  $\bc{\theta^t}_{t=1}^T$
	\end{small}}
	\end{algorithmic}
\end{algorithm}

We propose \algoname{}, presented in full in Algorithm \ref{alg:sparsefed}, by combining sparsification and norm clipping.
At each round of federated learning, each device downloads the current global model and computes an update on their local dataset.
This update is clipped according to a specified $\ell_2$ norm.
This controls $\rho$ and allows us to control the additive error.
The server aggregates all updates with a simple average.
The aggregated update is added to an error feedback vector. 
% \saeed{ We can say that the clipping part is for ensuring small $\rho$ and the top-k is for controlling the propagation...}
The server extracts the $top_k$ magnitude coordinates from the error feedback vector, and zeroes out these coordinates from the error feedback vector.
The $top_k$ coordinates are used to update the global model.
Because we update $k << d$ coordinates, we reduce the propagation error.

We first define a notion of sparsity for a protocol and use it to prove our main theorem. In Appendix \ref{sec:sparsefedsparse} we discuss why \algoname{} satisifies this notion. 
\begin{definition}[$(k,\gamma)$-sparsity]
A federated learning protocol $d=(\lambda, \cG, \cA)$ is $(k,\gamma)$-sparse on a dataset $D$ if for all $u_t=\cG(\theta_{t-1},D)$ generated during the process of training on $D$ $\cA(u_t)$ only has $k$ non-zero elements and we have
$$|\cA(u_t) - u_t|_1\leq \gamma.$$
\end{definition}

\begin{theorem}\label{theorem:main}
Let $f$ be a $c$-coordinatewise-Lipschitz and $(k,\gamma)$-sparse protocol on a dataset $D$. Let $w=min(d, 2k)$ then
$R(\rho) = \Lambda(T) (1+wc)^{\Lambda(T)}(\rho+2\gamma)$
is a certified radius for $f$.
\end{theorem}

In Theorem \ref{theorem:main} we improve the base term in propagation error term by a factor of $\frac{d}{2k}$, that can be multiple orders of magnitude.

% \saeed{the following sentence is a bit out-placed.}
\emph{In summary, \algoname{}  aggregates clipped updates from devices and only updates the $top_k$ coordinates of the aggregated update.
We show that the use of $top_k$ update sparsification improves the certified radius.}

\subsection{Efficiency and Convergence Analysis of \algoname{}}

\noindent \textbf{Convergence Analysis:} 
We show that \algoname{} converges as well as SGD in the base setting (e.g. when no attackers are present).
% \algoname{} is a compressed SGD algorithm with a biased gradient compression operator $\cC(\cdot)$ and error feedback, as are the methods of~\citep{karimireddy2019ef,zheng2019communication,ivkin2019communication}.
% Following prior work, we use the analysis of
% \citep{Stich2018SparsifiedSW}, who show that compressed SGD converges when $\cC$ is a $\delta$-contraction:
% \[\norm{\cC(\x)-\x}\leq(1-\delta)\norm{\x}\]
% Because our compression operator $\cC$ is global top-$k$, the contraction property follows~\citep{Stich2018SparsifiedSW}.
% Theorem \ref{theorem:convergence} follows Remark 4 of \citep{karimireddy2019ef} and shows that \algoname{} converges at asymptotically the same rate as SGD.
We make standard assumptions on the smoothness of the loss function and bounded gradient which are only necessary for our convergence analysis \citep{rothchild2020fetchsgd, karimireddy2019ef, xie2021crfl}.
% \saeed{make it clear that these assumptions are for convergence only. Also add citations for them being standard.}
% \prateek{consider adding another sentence, or state in brackets  that these assumptions are not necessary for our robustness analysis}

\begin{assumption}[Smoothness]\label{assumption:smoothness}
$\cL$ is $\ell$-smooth if $\forall x,y \in \cR^d \;\; |\cL(x) - (\cL(y) + \langle \nabla \cL(x), x-y \rangle)| \leq \frac{\ell}{2}\norm{x-y}_2^2 $
\end{assumption}

\begin{assumption}[Moment Bound]\label{assumption:momentbound}
For any x, our oracle returns \textbf{g} s.t.
$ \bbE{}[\textbf{g}] = \nabla \theta(x) \text{ and } \bbE{}\norm{\textbf{g}}_2^2 \leq \sigma^2 $
\end{assumption}

\begin{theorem}[Asymptotic Convergence of \algoname{}]\label{theorem:convergence}
For a protocol $f, \lambda(t) = \sqrt{t+1}^{-1}, \tau = 1, \cA = top_k, \cL$ satisfying Assumption \ref{assumption:smoothness}, 
% (the appropriate theorem from \citep{karimireddy2019ef} can be applied when Assumption \ref{assumption:smoothness} is not appropriate), 
$\cG$ satisfying Assumption \ref{assumption:momentbound}, we get the convergence rate of
\[ \min_{t \in T} \bbE{}[\norm{\nabla \cL(\theta_t)}] \leq \frac{4(\theta_0 - \theta_*) + \ell \sigma^2}{2 \sqrt{T+1}} + \frac{4 \ell^2 \sigma^2 (1 - \delta)}{\delta^2 (T+1)} \]
Therefore, $f$ converges asymptotically at the SGD rate.
\end{theorem}

\noindent \textbf{Communication efficiency of \algoname{}:} In practical deployments of federated learning systems, communication efficiency must be prioritized.
The $top_k$ sparsification used in \algoname{} requires communicating the full gradient at every iteration and therefore is not communication efficient.
\fedssgd{} is a communication efficient approximation of $top_k$ sparsification using the Count Sketch data structure \citep{rothchild2020fetchsgd}.
Because \fedssgd{} provably approximates the heavy hitter recovery properties of $top_k$ \citep{rothchild2020fetchsgd}, it inherits these robustness guarantees.
In Appendix \ref{appendix:fetchsgd}, we compare implementations of \algoname{} using both $top_k$ and \fedssgd{} and find that when using the latter, we are able to prove robustness and communication efficiency.

\vspace*{-13pt}
\section{EVALUATION}\label{sec: empirical}
\vspace*{-2pt}
We empirically demonstrate the effectiveness of our \algoname{} defense against strong attackers in a variety of realistic experimental settings. 
To this end, we set up the first environment to simulate model poisoning attacks on the cross-device setting of federated learning with tens of thousands of devices, aiming to emulate a real-world deployment as closely as possible.
In contrast, prior work has mostly evaluated attacks in the cross-silo setting with 10s to 100s of devices \citep{pmlr-v97-bhagoji19a,bagdasaryan18backdoor,wang2020attack,Fang2020LocalMP}.
We evaluate \algoname{} in both the cross-silo and cross-device settings against a breadth of attacks and find that we significantly outperform prior defenses.

\vspace*{-10pt}
% \vspace*{-5pt}
\subsection{Experimental setup}
\vspace*{-5pt}
All methods are implemented in PyTorch \citep{paszke2017automatic}.
We conduct experiments on computer vision (CIFAR10, CIFAR100, FashionMNIST, FEMNIST), and natural language processing (Reddit) datasets. 

Federated Extended MNIST (FEMNIST) dataset \citep{caldas2018leaf} is a dataset constructed specifically as a benchmark for federated learning.
Our goal is to train a model in a true federated fashion, i.e. we can only view each datapoint once.
We use a 40M-parameter ResNet101 for this task.
FEMNIST has 63 classes and a natural non-i.i.d.\ partitioning with an average of 226.83 datapoints for each of 3550 users, for a total of 805,263 datapoints.
Our goal is to simulate the cross-device setting as closely as possible, so we aim to have $\gtrsim 50$ devices participating in each round, with each device participating exactly once~\citep{kairouz2019advances}, without exceeding a batch size of $\approx 600$.
We split each user evenly into $9-10$ devices, yielding $35,000$ simulated devices and $35$ devices participating in each iteration.
Each device has a non-i.i.d.\ dataset that includes data from multiple classes.
% This partitioning allows us to train a model to convergence in just 1 pass over the dataset.

We also conduct experiments on Fashion MNIST (FMNIST) \citep{xiao2017fashionmnist}, CIFAR10/CIFAR100 \citep{krizhevsky2009learning}, that are benchmark tasks for computer vision.
We provide the experimental parameters in Table \ref{table:experimental-params} for the cross-silo and cross-device settings, for the number of devices $d$, number of devices participating at each iteration $w$, percentage of attackers $p$, and the auxiliary set size $s$: the number of datapoints we are attempting to modify model behavior on for the targeted model poisoning attack.
A key design choice is how to distribute the training data among simulated devices.
In the cross-silo setting, we simply distribute data i.i.d.\ across devices.
In the cross-device setting, we follow previous work \citep{rothchild2020fetchsgd} and artificially create non-i.i.d.\ datasets by giving each device images from only a single class.
% image classification datasets comprised of 60,000 $32\times32$ pixel color images distributed i.i.d.\ across 10 and 100 classes respectively (50,000/10,000 train/test split).
% To simulate the cross-device setting, we artificially create non-i.i.d.\ datasets by giving each client images from only a single class, and use this non-i.i.d.\ partitioning across all experiments.
% We use 10,000 devices, yielding 5 images per device for CIFAR10/CIFAR100 and 6 images per device for FMNIST.
% We choose this scale because small, non-iid local datasets are a core constraint of the cross-device setting \citep{kairouz2019advances} and this scale has been used in prior work on communication-efficient federated learning \citep{rothchild2020fetchsgd}.
At each round of federated learning, a subset of devices are randomly selected to participate.
% We vary the number of attackers.
Our 7M-parameter ResNet9 model architecture, data preprocessing, and most hyperparameters follow \citep{davidpage}. 
\begin{table}[ht]
    \centering
    \begin{tabular}{c|cc}
        Parameter & Cross-silo & Cross-device \\
        \toprule
        i.i.d.\ & TRUE & FALSE \\
        $d$ ($\#$ devices) & 1000 & 100000 \\
        $w$ ($\#$ participating) & 10 & 100 \\
        $p$ ($\%$ compromised) & 1 & 2 \\
        $a$ ($\bbE{}[\#]$ attackers per iter) & 0.1 & 2 \\
        $s$ (auxiliary set) & 50 & 500 \\
        $b$ (local batch size) & 50 & 5 \\
    \end{tabular}
    \caption{Parameters for CIFAR10, CIFAR100, MNIST, FashionMNIST in cross-silo and cross-device settings}
    \label{table:experimental-params}
\end{table}

\vspace*{-5pt}
\subsection{Attack details:}
We experiment with a number of attacks: targeted model poisoning, untargeted model poisoning, semantic backdoor, model replacement, colluding attack, and adaptive attack.
In all attacks, the attacker controls a number of devices and realizes the attack by uploading poisoning gradients to the server.
$p\%$ of the $d$ simulated devices are attackers.
We sample $w$ devices randomly at every iteration to participate, so we expect $a = p \cdot w$ devices to be compromised at each iteration.

\noindent \textbf{Targeted model poisoning:} 
We follow the attack procedure of \citep{pmlr-v97-bhagoji19a}.
We construct an auxiliary dataset of size $s$ with the following procedure: First, we sample $s$ points from the test distribution.
We then flip the label to one of the labels that is not the ground truth.
The objective of the attacker is to maximize the accuracy of the trained model on the auxiliary dataset (\emph{attack accuracy}), typically while ensuring that the model performance \emph{on the remaining data} does not degrade significantly.
The attacker is present throughout the course of training.

\noindent \textbf{Untargeted model poisoning attack:} 
Also known as a Byzantine attack, the attacker attempts to decrease the test accuracy of the trained model \citep{blanchard2017machine, mhamdi2018hidden}.
The attacker is present throughout the course of training, and succeeds when the model parameters diverge and can no longer be trained without resetting to an earlier checkpoint.

\noindent \textbf{Semantic backdoor via model poisoning:}
% \saeed{We can say semantic model poisoning or semantic attack to avoid confusion with backdoor attacks.}
We follow the backdoor attack described in \citep{sun2019backdoor}.
We train a model on FEMNIST and simulate 35,000 devices, 1000 of which are attackers.
We consider the semantic backdoor task of misclassifying the digit $7$ as $1$, creating $3000$ backdoors, the number of instances of the digit $7$ in the unperturbed validation set, and include results in Table \ref{table:eval}.
We include experiments that vary the semantic backdoor task in Appendix B.5.
% \prateek{a space saving way of making the suggested changes is to simply say: semantic backdoor via model poisoning in the title + use the keyword *unperturbed* when talking about validation set}

\noindent \textbf{Model replacement:}
In Appendix B.5 we evaluate \algoname{} against the model replacement attack of \citep{bagdasaryan18backdoor} on the Reddit dataset.
The attacker participates in a single iteration toward the end of training and scales their gradient so that they can entirely replace the trained global model.
In order to optimize for the $\ell_2$ norm clipping constraint, the attacker uses Projected Gradient Descent (PGD) with knowledge of the norm clipping parameter.
% We combine the model replacement attack with the targeted model poisoning attack, so that at each iteration the attacker constrains and scales their update.

\noindent \textbf{Colluding attack:}
We propose the colluding attack for the cross-device setting, where multiple attackers can be present in a single iteration.
The attackers collude by \emph{each} sending the same update.
In the cross-device setting, we combine the colluding attack with the targeted model poisoning attack, untargeted model poisoning attack, or semantic backdoor attack.

\vspace*{-13pt}
\subsection{SparseFed is an effective defense in the cross-silo setting}
\vspace*{-5pt}
We first evaluate \algoname{} in the cross-silo setting common to prior work to show the improvement of \algoname{} over the baseline $\ell_2$ clipping defense.
In Figures \ref{fig:crosssilokconvergence} and \ref{fig:crosssilokrobustness} we see that appropriately choosing $k$ allows us to mitigate the attack without harming convergence.
As we explain in Section \ref{sec: defense}, $\ell_2$ norm clipping is insufficient to mitigate the attack because minor perturbations at early iterations can propagate over the course of training.
This intuition validated by our results, that show that the use of norm clipping is not sufficient to deter the attacker.
From this, we can see the importance of coupling both norm clipping and update sparsification in \algoname{}.
The tradeoff that \algoname{} introduces for the attacker is forcing them to have large magnitude elements in order to have their component of the update appear in the $top_k$, however these are clipped due to the use of $\ell_2$ norm clipping, leading to ineffective attacks. 

\begin{figure*}[htb]
\minipage{0.24\textwidth}
  \includegraphics[width=\linewidth]{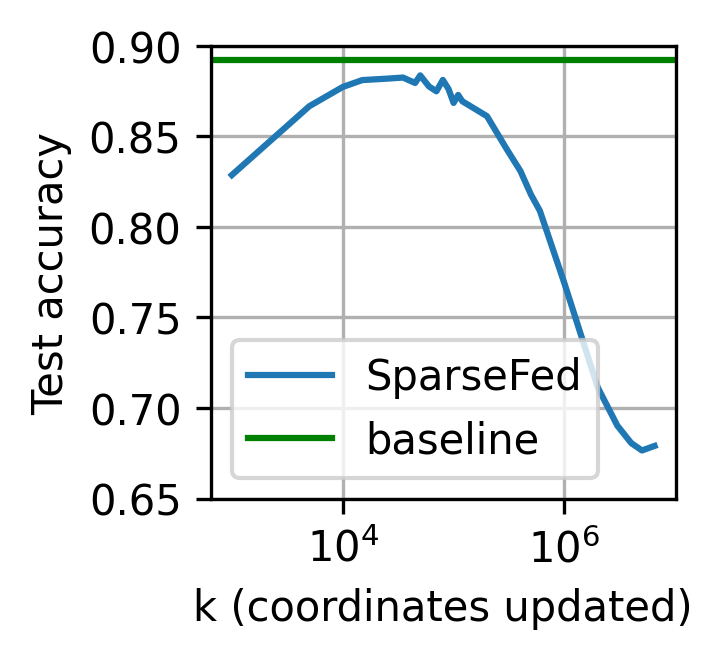}
  \subcaption{Impact of $k$ on convergence, cross-silo setting}
  \label{fig:crosssilokconvergence}
\endminipage\hfill
\minipage{0.24\textwidth}
  \includegraphics[width=\linewidth]{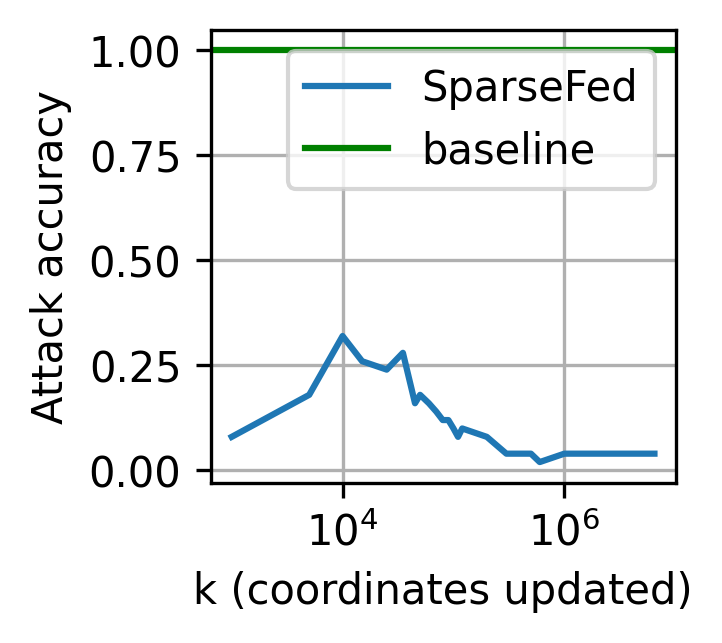}
  \subcaption{Impact of $k$ on robustness, cross-silo setting}
  \label{fig:crosssilokrobustness}
\endminipage\hfill
\minipage{0.24\textwidth}
  \includegraphics[width=\linewidth]{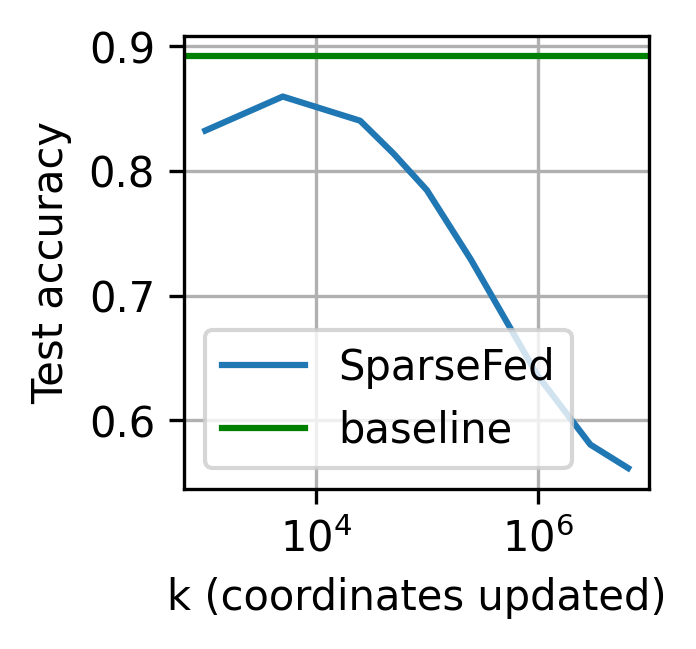}
  \subcaption{Impact of $k$ on convergence, cross-device setting}
  \label{fig:crossdevicekconvergence}
\endminipage\hfill
\minipage{0.24\textwidth}
  \includegraphics[width=\linewidth]{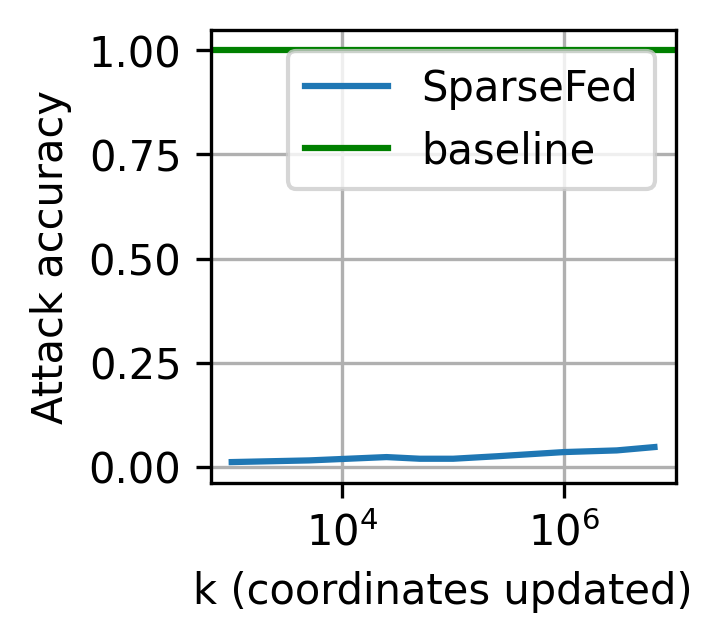}
  \subcaption{Impact of $k$ on robustness, cross-device setting}
  \label{fig:crossdevicekrobustness}
\endminipage
\caption{Tradeoffs between sparsification, convergence and robustness for the targeted model poisoning attack on CIFAR10 in the cross-silo and cross-device settings.}
\label{fig:ktradeoffs}
\end{figure*}

% \begin{figure}[H]
%     \centering
%     \includegraphics[width=\linewidth]{plots/smallscaleconvergencetradeoff.png}
%   \caption{When $k$ is correctly chosen, there is little degradation in test accuracy in the cross-silo setting.}
%     \label{fig:smallscaleconvergence}
% \end{figure}
% \begin{figure}[H]
%     \centering
%     \includegraphics[width=\linewidth]{plots/smallscalerobustnesstradeoff.png}
%   \caption{Most choices of $k$ do not permit the \textbf{targeted model poisoning} attack to succeed when compared to the $\ell_2$ clipping baseline in the cross-silo setting. \arjun{Where is the attack line for sparsefed? Also maybe use green for baseline?}}
%     \label{fig:smallscalerobustness}
% \end{figure}

\noindent \textbf{Impact of sparsification parameter $k$:} 
% \begin{figure}
%     \centering
%     \includegraphics[width=\linewidth]{plots/ksparsificationconvergencetradeoffcifar10.png}
%   \caption{When $k$ is correctly chosen ($k=5e3$), there is little degradation in test accuracy when there are no attackers}
%     \label{fig:kconvergence}
% \end{figure}
\algoname{} requires the sparsification parameter $k$.
We provide an adaptive algorithm for choosing $k$ in Appendix B.1.
% We initially pick $k = \frac{d}{T}$, where $d$ is the total number of model parameters and $T$ is the number of iterations per epoch, and tune $k$ from this starting point on the CIFAR10 dataset for a single device (i.e., the non-federated setting) to obtain the smallest possible value of $k$ that does not impact convergence.
When using ResNet9, we obtain a value of $k=1e3$ that does not significantly compromise convergence and use this across all datasets that use ResNet9 (FMNIST, CIFAR10, CIFAR100).
When using ResNet101, we obtain a value of $k=4e4$ and use this for all FEMNIST experiments.
% \begin{figure}
%     \includegraphics[width=\linewidth]{plots/ksparsificationconvergencerobustnesstradeoff.png}
%   \caption{When $k$ is correctly chosen ($k=5e3$)  the attack is mitigated for 200 attackers}
%   \label{fig:ktradeoff}
% \end{figure}
Fig. \ref{fig:ktradeoffs} shows the sparsification-utility-robustness tradeoff for the cross-silo and cross-device settings.
For small $k$ and large $k$ neither the attack nor the model converge.
When $k$ is too small, \algoname{} approaches a no-op as $k \rightarrow 0$.
When $k$ is too large, the use of momentum factor masking \citep{stich2018local, lin2017deep} prevents convergence to a benign optimum, which in turn makes it difficult for the attacker to perform model replacement~\citep{bagdasaryan18backdoor}.
Most choices of $k$ mitigate the attack, and the best choice of $k$ does not significantly degrade test accuracy.
We expect that practitioners will be able to easily tune the correct value of $k$ for their purpose, because the parameter can be tuned on a single device and does not need to be finetuned across datasets for the same architecture.

% \prateek{we mentioned 5 attacks just before this subsection. we should consider a forward pointer to the appendix to showcase results for our approach against all those attacks?}

\vspace*{-15pt}
\subsection{SparseFed is the most effective defense in the cross-device setting}
\vspace*{-5pt}
We next evaluate \algoname{} in the cross-device setting, which includes many more devices and the challenge of optimizing over small, non-i.i.d.\ datasets.
In Figures \ref{fig:crossdevicekconvergence} and \ref{fig:crossdevicekrobustness} we see that appropriately choosing $k$ allows us to mitigate the attack without harming convergence.
This is the setting that \algoname{} is designed for, and we evaluate it against prior work.

\noindent \textbf{Existing defenses cannot handle collusion}
Prior empirical defenses are designed under the assumption that data is distributed i.i.d.\ across devices and attackers do not collude amongst each other.
We carry out attacks in the cross-device setting, where data is non-i.i.d.\ and attackers have no restriction on their ability to collude, and conclude that \algoname{} is the only defense that maintains empirical robustness in this setting.
% \saeed{Can we have similar figures to 3.3 for cross device setting?}
% \prateek{agreed, the presentation here jumps to comparison with prior work very quickly, and figures similar to 3.3 would be very helpful. This would be a reader's first question when transitioning to a different setting}
In Table \ref{table:eval} we evaluate all defenses against a population of colluding attackers across all four datasets.
We report the attack accuracy; when a defense fails to converge, we mark it with DNC (this is discussed further below).
Bulyan and other Byzantine-resilient aggregation rules rely on eliminating outliers \citep{mhamdi2018hidden}.
Specifically, Bulyan determines outliers by measuring their distance from other updates in the population.
Because the attackers are colluding, their updates have a distance of $0$ from each other, and as a result Bulyan does not eliminate them.
Trimmed mean fails even against a single attacker, because trimmed mean relies on the assumption that a Byzantine attacker will either be the minimum or maximum value.
However, this assumption does not hold for a model poisoning attacker.
These conclusions are in line with conclusions from prior work \citep{pmlr-v97-bhagoji19a, Fang2020LocalMP, baruch2019little}.
Our experimental results demonstrate that even when attackers collude, they are unable to overcome the trade-off that is enforced by \algoname{}.

\noindent \textbf{Byzantine attacks:}
In Table \ref{table:table-byzantine-failure} we validate the effectiveness of \algoname{} against untargeted model poisoning attacks, or Byzantine attacks.
Byzantine attacks succeed more easily in the cross-device setting against prior defenses for the reasons mentioned above, but \algoname{} is still able to mitigate these.
% \prateek{need to engage more with these results}

\noindent \textbf{Impact of defenses on test accuracy}:
In Table \ref{table:table-defense-impact} we evaluate the impact of each defense on convergence in the absence of attacks.
Krum and coordinate median do not converge in the cross-device setting.
When Krum chooses a single model, it is overfitting the global model to the small local dataset of a single device.
Coordinate median does not converge because of the gap between median and mean.
Trimmed mean and Bulyan have a minor impact on test accuracy when the robustness parameter $f$ is small.
When 2 out of 100 devices are compromised, Bulyan will discard $4f+2=10$ gradients in order to maintain robustness.
For the challenging FEMNIST task, this information loss is too much and these methods do not converge.
These observations are in line with conclusions from prior work, that make the case for more complex algorithms \citep{Chen2020DistributedTW, MuozGonzlez2019ByzantineRobustFM, yin2018byzantine} that are out of the scope of this paper's evaluation.
Norm clipping acts as regularization and does not have much impact on the test accuracy.
As we show in Fig. \ref{fig:crossdevicekconvergence}, \algoname{} does not impact test accuracy significantly.

% \prateek{we need to rethink the organization of the main points here. we have paras like stealth of attack and strength of attack where quantitative results are in the appendix. we should consider moving these (and all other similar paras where ever possible) towards the very end of the section + minimize the real estate devoted to these points}

% \noindent \textbf{Krum and Median do not converge (Appendix B.3)}:
% This is due to the small and non-iid local datasets, that are an important facet of the cross-device setting of federated learning.
% Even when we make the data partitioning iid, Krum is still unable to converge because each device has at most 5-6 local datapoints, that cannot generalize to the data distributions, each of which has at least 10 classes.
\noindent \textbf{Verification of theory}:
In Section \ref{sec: defense} we analyze the certified radius of \algoname{}.
In Table \ref{table:eval} we provide observed distances between poisoned and benign models when using various defenses, and conclude that \algoname{} has both the lowest distance and lowest attack accuracy.
This verifies our theoretical guarantees.

\begin{table*}[ht]
  \caption{Krum, Bulyan, trimmed mean, coordinate median, norm clipping (clipping, $\ell_2=5$), and \algoname{} on FMNIST, CIFAR10, CIFAR100, and FEMNIST in the cross-device setting. 
%   (k=5e3,5e3,5e3,4e4, number of devices=1e4,1e4,1e4,3.5e4, number of attackers=200,200,200,1000 respectively).
  \algoname{} reduces the attack accuracy significantly more than other defenses.
  If a defense cannot converge we denote it with DNC.
  We report $\ell_1$ distances between poisoned and unattacked models at the end of training.
  \algoname{} has less than half the distance of the next best defense.}
  \label{table:eval}
  \centering
  \begin{tabular}{l|llll|ll}
    \toprule
    \multirow{2}{*}{\textbf{Defense}} &
    \multicolumn{4}{c}{\textbf{Attack Accuracy ($\%$) (Dataset)}} &
    \multicolumn{2}{c}{\textbf{Distance (thousands)}}\\
    % \cmidrule(r){2-5}
         & CIFAR10 & CIFAR100 & FMNIST & FEMNIST & CIFAR10 & FMNIST \\
    \midrule
    Trimmed Mean & 44.6 & 81.4 & 100 & DNC & 64 & 41 \\ 
    Bulyan & 36.2 & 81.8 & 100 & DNC & 68 & 39 \\
    Clipping & 100 & 100 & 100 & 100 & 73 & 40 \\
    \textbf{SparseFed (Ours)} & 4.6 & 23 & 2.2 & 2.86 & 31 & 16 \\
    \bottomrule
  \end{tabular}
\end{table*}

\begin{table*}[ht]
    \begin{subtable}{.5\textwidth}
    \caption{Comparison of Byzantine failure success rates on FashionMNIST. 
Ours: Cross-device setting. 
\citep{Fang2020LocalMP}: Numbers from their paper with 100 devices, 20 attackers (20 $\%$ compromised)}
\label{table:table-byzantine-failure}
      \centering
      \begin{tabular}{l|ll}
\toprule
\multirow{2}{*}{\textbf{Defense}} &
\multicolumn{2}{c}{\textbf{Test error}} \\
 & \textbf{Ours}& \textbf{\citep{Fang2020LocalMP}}\\
 \midrule
Krum & DNC & 87 \\
Median & DNC & 29 \\
Trimmed mean & 90 & 52 \\
Bulyan & 90 & 38 \\
\algoname{} & 20 & N/A \\
\bottomrule
\end{tabular}
    \end{subtable}
    \begin{subtable}{.5\textwidth}
    \caption{Comparing the impact on test accuracy of the defenses.
Cross-device setting, no attackers (averaged over 3 runs).}
\label{table:table-defense-impact}
      \centering
\begin{tabular}{lll}
\toprule
\textbf{Defense} & \textbf{Decrease}& \textbf{Test Acc}\\
 \midrule
No defense & 0 $\pm 0$ & 90.0 $\pm 0.1$ \\
$\ell_2$ & 2.0 $\pm 0.1$ & 88.0 $\pm 0.1$ \\
DP ($\sigma = 0.025$) & 20.0 $\pm 0.2$ & 70.50 $\pm 0.2$ \\
Krum & 80.0 $\pm 0$ & 10.0 $\pm 0$ \\
Median & 80.0 $\pm 0$ & 10.0 $\pm 0$ \\
Trimmed mean ($f=2$) & 9.23 $\pm 0.8$ & 80.77 $\pm 0.8$\\
Bulyan ($f=2$) & 9.56 $\pm 0.79$ & 80.44 $\pm 0.79$ \\
Bulyan ($f=10$) & 66.48 & 23.52 \\
\algoname{} ($k=1e3$) & 10.21 $\pm 0.7$ & 79.79 $\pm 0.7$ \\
\algoname{} ($k=5e4$) & 3.0 $\pm 0.01$ & 87.0 $\pm 0.01$ \\
\bottomrule
\end{tabular}
    \end{subtable}
  \end{table*}

\noindent \textbf{$\ell_2$ Norm clipping(Appendix \ref{appendix:normclipping})}:
We improve all defenses by combining them with $\ell_2$ norm clipping, and all results for defenses include norm clipping.
% When evaluating the proposed norm clipping and noise addition defense, we perform an ablation on the two components: norm clipping and noise addition.
% We conclude that norm clipping in isolation is sufficient as a defense, but present full results in the Appendix.
% The detailed results can be found in the Appendix, and 
% Because norm clipping can be combined with any of the other defenses, and we in fact use it in \algoname{}, we explore the benefits of combining norm clipping with all other defenses.
% We conclude that every defense benefits from implementing norm clipping at the device level.
% The detailed results can be found in the Appendix, and we conclude that every defense benefits from implementing norm clipping at the device level.
% All results presented for all other defenses in this section include norm clipping.
% \prateek{a reviewer might wonder if this was the reason for some defenses not converging? Ashwinee: how would the reviewer wonder this after they read the sentence that says "every defense benefits from implementing norm clipping"?}

\noindent \textbf{Hyperparameter tuning (Appendix \ref{appendix:hyperparametertuning})}:
We tune standard hyperparameters on the \fedavg{} baseline, and use these hyperparameters for all experiments. 
Krum, Bulyan and trimmed mean require the parameter $f$, the number of attackers present in the system.
% Krum, Bulyan, and trimmed mean require a meta-parameter as input, that is, the number of attackers the system needs to be robust to.
% We set this parameter to the exact number of attackers in the system for various numbers of attackers.
% This ensures that we are comparing against the strongest version of these defenses, even under the unrealistic assumption that the server has perfect knowledge of the attacker's strength (\algoname{} does not rely on such knowledge).
% We tune a number of hyperparameters for the other defenses and report results using the best values of these hyperparameters:
\fedavg{} requires the number of local epochs, a batch size for each epoch, and learning rate decay.
$\ell_2$ clipping requires the clip parameter.

\noindent \textbf{Stealth of attack (Appendix \ref{appendix:stealth})}:
We validate that the attack is stealthy when it succeeds, insofar as it does not compromise normal model operation significantly.
For the targeted model poisoning attack, the auxiliary dataset is divided equally across all classes.
Thus, the performance of any one class does not degrade significantly.
In the semantic backdoor attack, by definition the model fails on the class that is flipped by the semantic backdoor.

\noindent \textbf{Strength of attack (Appendix \ref{appendix:strength})}:
In order to validate our defense, we ensure that we test against the strongest available attacks.
We tune a number of attack parameters, record insights, and compare to the strongest attacks.
% We compare the attack accuracy of the colluding attack used in this work against prior attacks on Byzantine-resilient aggregation rules, and conclude that our attack is significantly more powerful than prior work considers.
% The key factor in the strength of the colluding attack is the ability for colluding attackers to send identical gradients and therefore avoid outlier detection by essentially vouching for each other.
We also include experiments using an adaptive attack that we design against \algoname{} which has perfect knowledge of the $top_k$ coordinates.
\vspace*{-13pt}
\section{Related Work}\label{sec: related_work}
\vspace*{-2pt}
\noindent \textbf{Federated learning:}
% The cross-device setting of federated learning features the complications introduced earlier: millions of devices with non-iid data distributions participating sparsely \citep{kairouz2019advances}.
% Generalizing results from prior work on hundreds of devices to the cross-device setting is challenging and still an open problem.
There are two main settings for federated learning: the cross-device setting and the cross-silo setting \citep{kairouz2019advances}.
The cross-device setting features all the complications we introduce earlier, namely hundreds of thousands or millions of devices with non-iid data distributions participating sparsely \citep{kairouz2019advances}.
% In the real world, devices generate datasets with sizes that follow a power law distribution \citep{pinto2012powerlaw}, so most devices will have small local datasets.
This is the original setting for federated learning, and it is the setting where we focus our analysis.
% The alternative cross-silo setting has received a great deal of interest in recent years, as it is the setting where banks or other organizations look to collaboratively train a model.
% In the cross-silo setting, typically $\leq 100$ organizations are regularly participating in training a model \citep{kairouz2019advances}.
Most prior work has operated on the scale of the cross-silo setting, with experiments on at most $100$ devices \citep{wang2020attack, fang2019local, pmlr-v97-bhagoji19a, bagdasaryan18backdoor}.
% However, generalizing results from prior work on hundreds of devices to cross-device settings of tens of thousands of devices is challenging and still an open problem.
% In this work, we conduct experiments at the scale of the cross-device setting on as many as $35,000$ simulated devices.

% Various optimization strategies have been proposed for fusing device updates in federated learning, each addressing specific efficiency issues:
% \texttt{FedCurvature}~\citep{fedcurv}, \texttt{FedMA}~\citep{fedma}, \texttt{FedProx}~\citep{fedprox}.
% \texttt{FedCurvature}~\citep{fedcurv} builds on lifelong learning algorithms in~\citep{forgetting} and is designed to handle catastrophic forgetting when training with non-iid data; \texttt{FedMA}~\citep{fedma} performs iterative layerwise model fusion with neuron matching reducing the overall communication overhead; \texttt{FedProx}~\citep{fedprox} generalizes and re-parameterizes \texttt{FedAvg}~\citep{mcmahan17fedavg} to stabilize training with non-iid data. 
% Finally, \texttt{FedAvg}~\citep{mcmahan17fedavg}, that we use in our work, simply performs an average of the device updates.
% Due to its simplicity and performance \texttt{FedAvg} has emerged as the de-facto optimization standard for federated learning deployments at scale~\citep{fed-scale}.

\noindent \textbf{Targeted model poisoning attacks:}
% There are three principal actors in a federated learning system: the server, benign devices and one or more attacker-controlled devices.
The goal of the attacker in a targeted model poisoning attack is to modify the model such that particular inputs induce misclassification~\citep{chen2017targeted, Biggio:2012:PAA:3042573.3042761, pmlr-v97-bhagoji19a, bagdasaryan18backdoor, wang2020attack}.
% The two main methods of backdooring the model are data poisoning and model poisoning \citep{chen2017targeted, Biggio:2012:PAA:3042573.3042761, pmlr-v97-bhagoji19a, bagdasaryan18backdoor, wang2020attack}; in this work, we focus on analyzing model poisoning attacks.
% We first define the \emph{auxiliary dataset}: a predetermined set of data that the attacker wants the model to specifically misclassify.
This can be a random set of data drawn from the validation distribution, with the labels randomly flipped to another class
% , e.g. some trucks become planes, other trucks become cats, and some cats become dogs in the CIFAR10 dataset 
\citep{bagdasaryan18backdoor, pmlr-v97-bhagoji19a}.
This can also be a semantic backdoor, wherein the attacker tries to flip the label of all data from a target class to another specific class, e.g. classifying all $1$s as $7$s in the MNIST dataset \citep{sun2019backdoor}.
% The objective of the attacker is to maximize the accuracy of the trained model on the auxiliary dataset (\emph{attack accuracy}), typically while ensuring that the model performance \emph{on the remaining data} does not degrade significantly.
% This second objective is referred to as stealth \citep{pmlr-v97-bhagoji19a}.
% In targeted model poisoning attacks \citep{pmlr-v97-bhagoji19a,bagdasaryan18backdoor,sun2019backdoor, Goldblum2020DatasetSF}, the attacker controls a number of devices, and sends poisoned gradients to the server.
% gradients computed over the auxiliary dataset to the server.
% The attacker boosts the magnitude of their gradient, ensuring they can insert a backdoor even after the server averages all aggregated gradients in the current iteration \citep{pmlr-v97-bhagoji19a,bagdasaryan18backdoor}.
% As we will discuss, proposed defenses make use of constraints to reduce the effectiveness of boosting \citep{sun2019backdoor}.
% In FedAvg, devices compute and apply multiple local rounds of stochastic gradient descent (SGD) on their local dataset \citep{mcmahan17fedavg}.
% The attacker also computes multiple local iterations, and when a constraint is present it uses projected gradient descent (PGD) to compute the final gradient \citep{sun2019backdoor}. 
\citep{wang2020attack} shows that backdoors sampled from the low-probability portion of the distribution can break existing defenses and are a byproduct of the existence of adversarial examples.
% / \citep{szegedy2013intriguing}.
% Prior work has evaluated three threat models in terms of the frequency of participation of attacker-controlled devices, all of which we analyze in this work: less than one attacker per round, one attacker per round, and multiple attackers per round \citep{sun2019backdoor, wang2019federated}.
% An important factor determining the strength of the attacker is the number of devices which are compromised.
% Prior work has used 1 out of 10 \citep{bagdasaryan18backdoor}, 20 out of 100 \citep{Fang2020LocalMP}, and specifically for the federated setting, 1 out of 30 devices being compromised \citep{sun2019backdoor}.
% In line with \citep{sun2019backdoor}, we restrict our analysis to the setting where the attacker controls $<3 \%$ of devices.
% However, generalizing $20 \%$ of devices being malicious to the scale of millions of devices is simply not realistic, and so we restrict our analysis to the setting where the attacker controls $< 5 \%$ of devices.

\noindent \textbf{Prior defenses:}
% FedAvg conducts a weighted average over all aggregated gradients, that has been shown to not be robust to even a single attacker~\citep{blanchard2017machine}.
The two main bodies of work on defenses against poisoning attacks are \emph{certified robustness} and \emph{empirical robustness}.
% We outline the most relevant prior work and show that \algoname{} surpasses all prior work in both certified and empirical robustness.

Ensemble methods have been proposed to certify robustness against poisoning attacks (e.g. \citep{jia2020intrinsic, cao2021provably}).
As we show in Section \ref{sec: defense}, ensemble methods do not scale to the cross-device setting because they rely on most subsamples not containing attackers.
This assumption breaks down for $n \gtrsim 10^4$, where it is very unlikely to randomly sample enough clients to train a good model without sampling an attacker.
Data poisoning defenses are insufficient against malicious clients that can manipulate model updates.
Provably secure defenses against data poisoning certify robustness in terms of the number of poisoned examples, but a single compromised device can poison an arbitrary number of their training datapoints, that breaks the core assumption of secure defenses for data poisoning.
\citep{levine2021deep} partition the training dataset with a hash function for certified robustness, but their defense is only applicable to deterministic training algorithms for data poisoning.
Works that use randomized smoothing at testing time~\citep{rosenfeld2020certified, wang2020certifying} are complementary to our work, that is a procedure solely for training.
\citep{xie2021crfl} use noise during training and provide an inference-time smoothing procedure to certify robustness in federated learning.
However, their goal is to finetune an already poisoned model to erase backdoors.
There are a number of defenses that provide empirical robustness against poisoning attacks.
In our evaluation at scale we also compare to five of these empirical defenses, each of which we adapt and improve for the federated setting: trimmed mean \citep{yin2018byzantine}, median \citep{yin2018byzantine}, Krum \citep{blanchard2017machine}, Bulyan\citep{mhamdi2018hidden}, and norm clipping \citep{sun2019backdoor}.
We provide exact algorithms for these defenses in Appendix B.1.
These defenses only achieve provable guarantees with an i.i.d.\ assumption on the device data distributions \citep{mhamdi2018hidden}, that is not valid in the cross-device setting. 
Further, they have been shown to be ineffective against model poisoning attacks~\citep{attack-2019, pmlr-v97-bhagoji19a, xie2020dba, Fang2020LocalMP} in practice, e.g. by crafting updates that do not significantly differ from benign updates \citep{pmlr-v97-bhagoji19a}.

\vspace*{-13pt}
\section{Discussion}
\vspace*{-2pt}
Prior work in poisoning attacks on federated learning has demonstrated that existing defenses are vulnerable to attacks.
We complement this body of work by introducing \algoname{}, an optimization algorithm for federated learning which combines update sparsification and norm clipping.
We prove a certified radius for \algoname{} that improves over baseline federated learning.
\algoname{} does not significantly decrease test accuracy, and mitigates attacks in both the cross-silo and cross-device settings.
We evaluate \algoname{} empirically against existing defenses, and confirm that it outperforms these against multiple strong attacks.
\appendix
\onecolumn
\aistatstitle{SparseFed: Supplemental Material}
The Appendix is organized as follows:
\begin{itemize}
    \item Appendix \ref{appendix:theory} gives full proofs of the theorems in the main body.
    \begin{itemize}
        \item Appendix \ref{appendix:sparsefedsparse} the proof of the main certified radius theorem
        \item Appendix \ref{appendix:convergence} the convergence analysis of the defense
        \item Appendix \ref{appendix:singlelayer} full computation of Lipschitz constant of a single layer network.
        \item Appendix \ref{appendix:computingradius} procedure for computing certified radius
    \end{itemize}
    \item Appendix \ref{appendix:methods} gives details on the methods and metrics used throughout the main body of the paper and the Appendix.
    \begin{itemize}
        \item Appendix \ref{appendix:fedavg} \fedavg{}
        \item Appendix \ref{appendix:attack} the attack
        \item Appendix \ref{appendix:byzantinedefenses} Krum, Bulyan, trimmed mean, coordinate median
        \item Appendix \ref{appendix:sparsefed} \algoname{} implemented with true top-$k$ and \fedssgd{}
        \item Appendix \ref{appendix:adaptivekselection} an adaptive algorithm for selecting $k$ in \algoname{}.
        \item Appendix \ref{appendix:oif} the metrics used throughout the main body and Appendix.
    \end{itemize}
    \item The rest of Appendix B gives further experimental results for the conclusions reached in the main body of the paper.
    \begin{itemize}
        \item Appendix \ref{appendix:normclipping}  the use of $\ell_2$ norm clipping in \algoname{} and prior defenses.
        \item Appendix \ref{appendix:hyperparametertuning} the full range and results of hyperparameters tuned.
        \item Appendix \ref{appendix:defenseimpact} the impact of each defense on convergence.
        \item Appendix \ref{appendix:stealth} the stealth of the attack.
        \item Appendix \ref{appendix:strength} validates that we are evaluating \algoname{} against the strongest available attack.
        % \item Appendix \ref{appendix:theoryimplications} connects the theory to realized bounds on model drift.
        \item Appendix \ref{appendix:rangeproofs} the compatibility of \algoname{} with secure aggregation.
        \item Appendix \ref{appendix:attacktuning} the parameters of the attack and how they are tuned.
        \item Appendix \ref{appendix:fetchsgd} the case for \algoname{} implemented with \fedssgd{} as an algorithm which achieves security and communication efficiency.
    \end{itemize}
    \item Appendix \ref{appendix:limitations} discusses the limitations and societal impact of our work.
\end{itemize}
\section{Proofs}\label{appendix:theory}
\newcommand{\mythm}[4]{{\textbf{#1~#2~}}{(#3)}{\textit{#4}}}
\subsection{Propagation analysis of sparse aggregation}\label{appendix:sparsefedsparse}
Here we prove Theorems \ref{theorem:main} and \ref{theorem:framework}.
Before that, we introduce several definitions that will be used in stating and proving the Theorem.

\noindent \textbf{Notation:}
Let $Z$ be the data domain and $D^t$ be data sampled (not necessarily i.i.d.) from $Z$ at iteration $t$.
Let $\Theta$ be the class of models in $d$ dimensions, and $\cL: \Theta \times Z^* \rightarrow \cR$ be a loss function. 
% \supriyo{Is $Z^*$ for a possibly empty domain?}
A protocol $f=(\cG,\cA, \lambda)$ consists of a gradient oracle 
$\cG(\theta,D,t) \rightarrow \cR^d$ that takes a model, a dataset and a round index and outputs the update vector $u^t$. $f$ also includes an update algorithm $\cA:u^t \in \cR^d \rightarrow \cR^d$, e.g. momentum.
% \supriyo{Should it be $\lambda(t) \in \cR$} 
$\lambda(t) \in \cR$ is a learning rate scheduler, possibly static, and $\Lambda(t)$ the cumulative learning rate $\Lambda(t)=\sum_{i=1}^t\lambda(t)$. 
The update rule of the protocol is then defined as $\theta_{t+1} = \theta_{t} - \lambda(t) \cA(u^t)$.

\mythm{Definition}{1}{Poisoning Attack [Restated]}{
For a protocol $f=(\cG,\cA, \lambda)$ we define the set of \textbf{poisoned} protocols $F(\rho)$ to be all protocols $f^*=(\cG^*,\cA, \lambda)$ that are exactly the same as $f$ except that the gradient oracle $\cG^*$ is a $\rho$-corrupted version of $\cG$. 
That is, for any round $t$ and any model $\theta_t$ and any dataset $D$ we have we have $\cG^*(\theta_t, D)= \cG(\theta_t, D) +\epsilon$ for some $\epsilon$ with $||\epsilon||_1 \leq \rho$.
}

\mythm{Definition}{3}{Coordinate Lipschitz [Restated]}{
A protocol $f(\cG,\cA, \lambda)$ is $c$-coordinatewise Lipschitz if for any round $t\in[T]$, models $\theta_t, \theta^{*}_t \in \cM$, and a dataset $D$ we have that the outputs of the gradient oracle on any coordinate cannot drift too much farther apart. Specifically, for any coordinate index $i\in[d]$
$$\Big|\cG(\theta^*_t, D)[i] - \cG(\theta_t,D)[i]\Big|\leq c\cdot |\theta^{*}_t-\theta_t|_1.$$}

\mythm{Definition}{4}{$(k,\gamma)$-sparsity [Restated]}{
A federated learning protocol $d=(\lambda, \cG, \cA)$ is $(k,\gamma)$-sparse on a dataset $D$ if for all $u_t=\cG(\theta_{t-1},D)$ generated during the process of training on $D$ $\cA(u_t)$ only has $k$ non-zero elements and we have
$$|\cA(u_t) - u_t|_1\leq \gamma.$$
}
We will use this definition in our following Theorem. 
In Subsection \ref{sec:sparsefedsparse} we explore the sparsity of the SparseFed algorithm.

\mythm{Definition}{2}{Certified radius [Restated]}{
Let $f$ be a protocol and $f^*\in F(\rho)$ be a poisoned version of the same protocol.
Let $\theta_T, \theta^{*}_T$ be the benign and poisoned final outputs of the above protocols on a dataset $D$.
We call $R$ a certified radius for $f$ on a dataset $D$ if 
$\forall f^*\in F(\rho);  R(\rho) \geq   |\theta_T - \theta^{*}_T|_1.$
}

\begin{theorem}\label{general}
Let $f$ be a $c$-coordinatewise-Lipschitz and $(k,\gamma)$-sparse protocol on a dataset $D$. Let $w=min(d, 2k)$ then
$R(\rho) = \Lambda(T) (1+wc)^{\Lambda(T)}(\rho+2\gamma)$
is a certified radius for $f$.
\end{theorem}
Before proving the above theorem, note that that Theorem \ref{general} immediately implies Theorems \ref{theorem:framework} and \ref{theorem:main}.
\begin{proof}
Let $f^*=(\cG^*,\cA, \lambda)\in f(\rho)$ be an arbitrary $\rho$-poisoned version of $f$.
We first define two sequence of models $(\theta_b^0, \dots, \theta_b^T)$ and $(\theta^0, \dots, \theta^T)$ where $\theta_b^t$ is the model trained in the first $t$ iterations through the benign (non-poisoned) gradient oracle $\cG$ and $\theta^t$ is the model trained in the first $t$ iterations through a $\rho$ poisoned aggregation $\cG^*$. 
Also, we define $u_b^1,\dots,u_b^T$ and $u^1,\dots,u^t$ to be the update vectors that the benign oracle $\cG$ would produce on models $\theta_b^1,\dots,\theta_b^{T-1}$ and $\theta^1,\dots,\theta^{T-1}$, respectively. 
We also define $\hat{u}^1,\dots, \hat{u}^T$ to be the output of the adversarial gradient oracle $\cG^*$ on models $\theta_1,\dots,\theta_{T-1}$. By the definition of $\rho$-poisoning, we have
$|\hat{u}^t - u^t|_1 \leq \rho.$

Note that by the definition of coordinatewise Lipschitzness, for any coordinate $i\in[d]$ we have
$$|u^t[i] - u^t_b[i]| \leq c|\theta^{t-1}- \theta_b^{t-1}|_1.$$
Now, we use the triangle inequality to connect the distance between $\theta^t$ and $\theta^t_{b}$ to that of the previous round as follows
\begin{align}\label{eq1}
\Big|\theta^t - \theta^t_{b}\Big|=\Big|\theta^{t-1}[i] - \lambda(t)\cA(\hat{u}^{t}) - \theta_b^{t-1} + \lambda(t)\cA(u_b^{t})\Big|\leq \Big|\theta^{t-1}- \theta_b^{t-1} \Big| + \lambda[t]\Big|\cA(\hat{u}^{t}) - \cA(u_b^{t})\Big| \end{align}

Now we prove the following Lemma that bounds the difference between updates on the benign and poisoned models.
\begin{lemma}\label{lem:sparseAgg}
We have
\begin{align*}
    &|\cA(\hat{u}^{t})-\cA(u_b^{t})|_1 \leq \sum_{i\in I}|(\hat{u}^{t}[i]-u_b^{t}[i])| + 2\gamma 
\end{align*}
where $I=\set{j\in [d] \text{~s.t.~} \cA(\hat{u}^{t})[j]\neq 0 \text{~or~}  \cA(u_b^{t})[j]\neq 0}$.
\end{lemma}
\begin{proof}
Let $\tau_1$ and $\tau_2$ be two vectors such that $\tau_1[i]=1$ if $\cA(\hat{u}^{t})[i]\neq0$ and $\tau_1[i]=0$ otherwise. Similarly, $\tau_2[i]=1$ if $\cA(u_b^{t})[i]\neq0$ and  $\tau_2[i]=0$ otherwise. Let $I'$ be the locations where $\tau_1[i]=1$ and $\tau_2[i]=1$. Now we have

\begin{align*}
    |\cA(u^{t}) -\cA(u_b^{t})|
    &= \sum_{i=1}^d |\hat{u}^{t}[i]\tau_1[i] - u_b^{t}[i]\tau_2[i]| \\
    &= \sum_{i\in I'}|(\hat{u}^{t}[i]-u_b^{t}[i])| + \sum_{i\in I\setminus I' } |\hat{u}^{t}[i]\tau_1[i] - u_b^{t}[i]\tau_2[i]|\\
    &\leq \sum_{i\in I'}|(\hat{u}^{t}[i]-u_b^{t}[i])| + \sum_{i \in I\setminus I' } |\hat{u}^{t}[i]\tau_1[i] - u_b^{t}[i]\tau_1[i]|\\
    &~~+\sum_{i \in I\setminus I' }|\hat{u}^{t}[i]\tau_2[i] - u_b^{t}[i]\tau_2[i]|+ \sum_{i \in I\setminus I' }|u_b^{t}[i]\tau_1[i] - \hat{u}^{t}[i]\tau_2[i] |\\
    &= \sum_{i\in I'}|(\hat{u}^{t}[i]-u_b^{t}[i])| + \sum_{i \in I\setminus I' } |\hat{u}^{t}[i] - u_b^{t}[i]|(\tau_1[i] + \tau_2[i])+ \sum_{i \in I\setminus I' }|u_b^{t}[i]\tau_1[i] - \hat{u}^{t}[i]\tau_2[i] |\\
    &= \sum_{i\in I'}|(\hat{u}^{t}[i]-u_b^{t}[i])| + \sum_{i \in I\setminus I' } |\hat{u}^{t}[i] - u_b^{t}[i]|+ \sum_{i \in I\setminus I' }|u_b^{t}[i]\tau_1[i] - \hat{u}^{t}[i]\tau_2[i] |\\
    &= \sum_{i\in I}|(\hat{u}^{t}[i]-u_b^{t}[i])| + \sum_{i \in I\setminus I' }|u_b^{t}[i]\tau_1[i] - \hat{u}^{t}[i]\tau_2[i] |\\
    &\leq  \sum_{i\in I}|(\hat{u}^{t}[i]-u_b^{t}[i])| +  \sum_{i\in I\setminus I' } |\hat{u}^{t}[i]\tau_2[i]| + |u_b^{t}[i]\tau_1[i]|\\
    &= \sum_{i\in I}|(\hat{u}^{t}[i]-u_b^{t}[i])| + \sum_{i\in I\setminus I' } |\hat{u}^{t}[i](1-\tau_1[i])|+ \sum_{i\in I\setminus I' } |u_b^{t}[i](1-\tau_2[i])|\\
    &\leq \sum_{i\in I}|(\hat{u}^{t}[i]-u_b^{t}[i])|+  |\hat{u}^{t} -\cA(\hat{u}^{t})| + |u_b^{t}-\cA(u_b^{t})|\\
    &\leq \sum_{i\in I}|(\hat{u}^{t}[i]-u_b^{t}[i])| + \gamma + \gamma.
\end{align*}
which finishes the proof.
\end{proof}
Now, based on Lemma \ref{lem:sparseAgg}, the $(k,\gamma)$ sparsity of $f$, the Lipschitzness, and since $|I| \leq w$ we conclude that

\begin{align}\label{eq2}
    &|\cA(\hat{u}^{t})-\cA(u_b^{t})|_1 \leq \sum_{i\in I}|(u^{t}[i]-u_b^{t}[i])| + \sum_{i\in I}|(\hat{u}^{t}[i]-u^{t}[i])|+ 2\gamma\leq  wc|\theta^{t-1} + \theta_b^{t-1}| + \rho + 2\gamma.
\end{align}
By plugging this into Equation \ref{eq1} we get
\begin{align}\label{eq3}
\Big|\theta^t - \theta^t_{b}\Big|\leq (1+wc\lambda(t))\Big|\theta^{t-1} - \theta^{t-1}_{b}\Big| +(\rho+2\gamma)\lambda(t).
\end{align}
Now using this equation, we inductively prove the Theorem. Assume for $T-1$ the statement of theorem holds. By Equation \ref{eq3} and the induction hypothesis we have
\begin{align}\label{eq4}
\Big|\theta^T - \theta^T_{b}\Big|\leq (1+wc\lambda(T))\Lambda(T-1)(1+wc)^{\Lambda(T-1)}(\rho+2\gamma) +(\rho+2\gamma)\lambda(T).
\end{align}
Then, by Bernouli's inequality we have 
\begin{align*}\Big|\theta^T - \theta^T_{b}\Big|&\leq \Lambda(T-1)(1+wc)^{\Lambda(T-1) + \lambda(T)}(\rho+2\gamma) +(\rho+2\gamma)\lambda(T)\\ &=\Lambda(T-1)(1+wc)^{\Lambda(T)}(\rho+2\gamma) +(\rho+2\gamma)\lambda(T)\\
&\leq (\Lambda(T-1) + \lambda(T))(1+wc)^{\Lambda(T)}(\rho+2\gamma)\\
&\leq \Lambda(T)(1+wc)^{\Lambda(T)}(\rho+2\gamma) .
\end{align*}
And this finishes the proof.
\end{proof}
\begin{remark}[How does sparsity help robustness?]
In our analysis of the effect of sparsity on the certified radius, we first proved Lemma \ref{lem:sparseAgg} to show that the effect of poisoning at each iteration is bounded by $\rho + 2\gamma$. 
Note that if we just use the identity aggregation (which is not sparse), we would get a better bound of $\rho$ for each iteration. 
Then, how are we getting a better final bound with sparsity? 
We emphasize that the goal of sparsity is to bound the "propagation error" during the entire training. 
The improved bound is achieved because of the fact that sparsification removes most of the noise that that poisoning can cause on the updates of benign parties. 
As we see in Table \ref{table:eval}, our approach can actually reduce the final distance between adversarial and benign models which verifies our theory and shows the importance of considering the propagation error. 
\end{remark}

\subsubsection{SparseFed is a sparse protocol}\label{sec:sparsefedsparse}
The definition of sparsity requires that the aggregation protocol to only update $k$ coordinates. The $top_k$ operator, by definition, only updates $k$ operators. The only thing that remains is to show that SparseFed can achieve a small $gamma$ as well. Here, we bound the $gamma$ for SparseFed, given a certain loss rate that is a known a priory.

\begin{definition}\label{def:omega}[loss rate $\omega_k$ for top-$k$ operator]
Let $\omega_k$ be the fraction of $l_1$ mass of information lost via $top_k$, where $top_k(u)$ recovers a $1-\omega_k$ fraction of the $l_1$ mass of $u$.
% In particular, $1-\omega_k$ fraction of the $l_1$-norm mass of the gradients should be recovered after top-$k$ is applied. 
For any model $M$, any $i\in [T]$ and update vector $(u^1,\dots,u^n)$ calculated by all parties (including benign and adversarial gradients), and memory $W$, we have:
\begin{align}
    |top_k(u^t + W^t)|_1\geq (1-\omega_k)|u^t +W^t|_1.
\end{align}
When clear from the context, we use $\omega$ instead of $\omega_k$.
\end{definition}
We first show that the size of memroy vector $W$ is bounded.
\begin{lemma}\label{lem:memory}
Let $W_t$ and $W_t^b$ be the memory vector at round $t$ for the benign and poisoned protocol respectively. After each iteration we have 
$$|W^i|\leq L\sqrt{d}\cdot \frac{\omega}{1-\omega}$$
and 
$$|W_b^i|\leq L\sqrt{d}\cdot \frac{\omega}{1-\omega}$$
where $L$ is the $\ell_2$ clipping threshold.
\end{lemma}
\begin{proof}
We prove this by induction on $i$. The proof is similar for $W_i$ and $W^i_b$ so we only prove it for $W^i$. For $i=0$ the induction hypothesis is correct. Now assume the hypothesis is correct for round $i-1$, namely
$$|W^{i-1}| \leq L\sqrt{d}\cdot \frac{\omega}{1-\omega}.$$
For round $i$ we have 
\begin{align*}|W^i|&=|W_{i-1} + u_{i-1} -top_k(W_{i-1} + u_{i-1})|\leq \omega(|W_{i-1} + u_{i-1}|)\leq \omega( L\sqrt{d}\cdot \frac{\omega}{1-\omega} + L\sqrt{d} )=  L\sqrt{d}\cdot \frac{\omega}{1-\omega}
\end{align*}
which finishes the proof.
\end{proof}
Now we show that after applying $top_k$ and memory, we do not deviate much from the original gradient (i.e. $\gamma$ is small).
\begin{lemma}
Let $\gamma=2L \sqrt{d} \frac{\omega}{1-\omega}$, we have
$$|top_k(u_t+ W) - u_t|_1 \leq \gamma.$$
\end{lemma}
\begin{proof}
Given that the loss rate of the $top_k$ is $\omega$, we have
$$|top_k(u_t+ W) -u_t - W|_1\leq \omega|u_t+W| \leq \omega(|u_t| + |W|) \leq L \sqrt{d} \frac{\omega}{1-\omega}.$$
Therefore, we have
$$|top_k(u_t+ W) -u_t|_1 \leq |top_k(u_t+ W) -u_t - W|_1  + |W|_1\leq 2L \sqrt{d} \frac{\omega}{1-\omega}.$$
\end{proof}

\subsection{Convergence analysis of \algoname{}}\label{appendix:convergence}
We first restate the convergence of Error Feedback SGD (EF-SGD) of \citep{karimireddy2019ef} and then analyze \algoname{} under this framework.

\subsubsection{Analysis of Error Feedback SGD}
\begin{algorithm}[H]
\caption{EF-SGD}
\label{alg:ef}
    \begin{algorithmic}
    {\begin{small}
	\REQUIRE learning rate $\gamma$, compressor $\cC(\cdot), x_0 \in \cR^d$
	\STATE $e_0 = 0 \in \cR^d$
	\FOR{$t = 0, \cdots, T-1$}
	\STATE $g_t :=$ stochasticGradient($x_t$)
	\STATE $p_t := \gamma g_t + e_t$
	\STATE $\delta_t := \cC(p_t$
	\STATE $x_{t+1} := x_t - \delta_t$
	\STATE $e_{t+1} := p_t - \delta_t$
	\ENDFOR
	\end{small}}
	\end{algorithmic}
\end{algorithm}
\begin{assumption}[Compressor]\label{assumption:efcompression}
An operator $\cC: \cR^d \rightarrow \cR$ is a $\delta$-approximate compressor over $\cQ$ for $\delta \in [0,1]$ if \[\norm{\cC(x) - x}_2^2 \leq (1 - \delta) \norm{x}_2^2, \forall x \in \cQ \]
\end{assumption}
\begin{assumption}[Smoothness]\label{assumption:efsmoothness-2}
A function $f: \cR^d \rightarrow \cR$ is L-smooth if for all $x, y \in \cR^d$ the following holds: \[ |f(x) - (f(x) + \langle \nabla f(x), y-x \rangle)| \leq \frac{L}{2}\norm{y-x}_2^2 \]
\end{assumption}
\begin{assumption}[Moment Bound]\label{assumption:efmomentbound-2}
For any x, our query for a stochastic gradient returns g such that \[\bbE{}[g] = \nabla f(x) and \bbE{}\norm{g}_2^2 \leq \sigma^2 \]
\end{assumption}
\begin{theorem}[Non-convex convergence of EF-SGD]\label{theorem:ef}
Let ${x_t}_{t \geq 0}$ denote the iterates of Algorithm \ref{alg:ef} for any step-size $\gamma > 0$. Under Assumptions \ref{assumption:efcompression}, \ref{assumption:efsmoothness-2}, \ref{assumption:efmomentbound-2}, \[\min_{t \in T} \bbE{}[\norm{\nabla f(x_t)}^2] \leq \frac{2(f(x_0) - f^*)}{\gamma(T+1)}+ \frac{\gamma L \sigma^2}{2} + \frac{4 \gamma^2 L^2 \sigma^2 (1 - \delta)}{\delta^2} \]
\end{theorem}
\subsection{Analysis of \algoname{}}
To prove the convergence of \algoname{}, we simply use Theorem \ref{theorem:ef} and prove that the necessary assumptions are satisfied.
That is, we prove that \algoname{} fits into the theoretical framework of \citep{karimireddy2019ef}.

We know already that the top-$k$ operator we use is a $\delta$-approximate compressor \citep{karimireddy2019ef}, which satisfies the first assumption.
The second and third assumptions, we directly reproduce for the gradient oracle that represents the individual device gradients.

\begin{assumption}[Smoothness]\label{assumption:smoothness-2}
$\cL$ is $\ell$-smooth if $\forall x,y \in \cR^d \;\; |\cL(x) - (\cL(y) + \langle \nabla \cL(x), x-y \rangle)| \leq \frac{\ell}{2}\norm{x-y}_2^2 $
\end{assumption}

\begin{assumption}[Moment Bound]\label{assumption:momentbound-2}
For any x, our oracle returns \textbf{g} s.t.
$ \bbE{}[\textbf{g}] = \nabla \theta(x) \text{ and } \bbE{}\norm{\textbf{g}}_2^2 \leq \sigma^2 $
\end{assumption}

Because \algoname{} is essentially EF-SGD for federated learning, it only remains to show that the federated setting does not complicate our analysis.
The federated setting comes with the complications of LocalSGD, namely multiple local epochs, and the non-i.i.d.\ distribution of data across devices.

As per the statement of Theorem \ref{theorem:convergence}, we only prove guarantees for $\tau=1$; that is, we only do a single local epoch.
Prior work has evidenced the challenges of analyzing convergence of LocalSGD in the presence of non-i.i.d.\ data \citep{li2019convergence}, and we find empirically that multiple local epochs are unfavorable for both convergence and robustness in the cross-device setting that is the focus of our work.

Therefore, \algoname{} directly fits into the theoretical framework of \citep{karimireddy2019ef} and Theorem \ref{theorem:ef} proves the convergence of \algoname{}.

In Fig \ref{fig:convergence} we empirically validate the speed of convergence of \algoname{} and find that it converges at the same rate as \fedavg{}, even in the presence of attackers.

\begin{figure}
    \centering
    \includegraphics{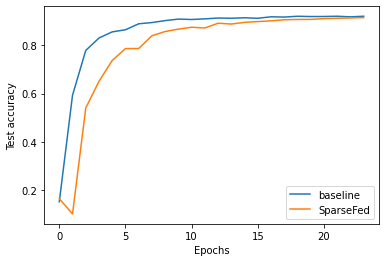}
    \caption{\algoname{} converges at the same rate as the baseline (FedAvg) on CIFAR10 in the cross-device setting}
    \label{fig:convergence}
\end{figure}

\subsubsection{Training a single layer neural network with SGD}\label{appendix:singlelayer}

\begin{example}[Training a single layer neural network with SGD]
% In this example, we show that the coordinatewise Lipschitz constant of the SGD protocol for a single layer neural network is $c=\frac{1}{4}$.
% We provide the full computation in Appendix B.1.
% \begin{equation}
%     \sup_{x, \theta_1, \theta_2} | g(x, \theta_1)_{i^{*} j^{*}} - g(x, \theta_1)_{i^{*} j^{*}} | \leq \frac{1}{4} |\theta_1 - \theta_2|
% \end{equation}
In this example, we compute the coordinatewise Lipschitz constant of the SGD protocol for a single layer neural network defined as $\sigma(\theta x)$, where $\sigma$ is the softmax function and $\theta \in \cR^d$ are the network parameters. For cross-entropy loss-based training using dataset $D$, we show that the constant $c=\frac{1}{4}$. Formally,
\begin{equation}
\sup_{D \in Z, \theta_1, \theta_2 \in \cM} | \cG(\theta_1, D)[ i ]  - \cG(\theta_2, D)[ i ] |_1 \leq \frac{1}{4} | \theta_1 - \theta_2 |_1 \mbox{ for any coordinate index } i \in [d] \nonumber
\end{equation}
%where $\cG(\theta, D)[i] = \frac{\partial \mathcal{L}}{\partial \theta_i}$. %We provide the full computation in Appendix B.1.
\end{example}

Without loss of generality, we assume that dataset $D$ is comprised of samples of the form $(x, y)$, where $x \in [0, 1]^m$, and $y \in \{0, 1\}^C$ is the one-hot encoded representation of any of the $C$ classes. 
For the single layer neural network, the model parameters are denoted by $\theta \in \cR^{C \times m}$, and the softmax layer by the function $\sigma(\cdot)$. The neural network can thus be represented as $\Phi(x, \theta) = \sigma(\theta x)$.

We define $g(\theta, x) = \frac{\partial \mathcal{L}(\Phi(x, \theta), y)}{\partial \theta}$ where $\mathcal{L}$ is the softmax cross entropy loss function. For the SGD protocol, $\cA(u) = u$, and $\cG(\theta, D) = g(\theta, x)$. Our goal is to find a Lipschitz constant $L$ such that, for all indices $i \in [C]$ and $j \in [m]$,
\begin{equation}
\sup_{x \in D, \theta_1, \theta_2} \frac{| g(\theta_1, x)_{ij} - g(\theta_2, x)_{ij} |_1}{| \theta_1 - \theta_2 |_1} \leq L
\label{eqn:lc_0}
\end{equation}
We define intermediate variable $z = \theta x$ and the neural network output distribution $p = \sigma(z)$, such that both $p, z \in \mathbb{R}^C$. Note, for a given target class $t$, the cross entropy loss function $\mathcal{L}(p, y) = -\log {p_t}$ where $p_t = \frac{e^{z_t}}{\sum_{j}e^{z_j}}$. Thus, 
\begin{equation}
    g(\theta, x)_{ij} = \frac{\partial \mathcal{L}}{\partial \theta_{ij}} = \sum_{c=1}^{C} \frac{\partial \mathcal{L}}{\partial z_c} \frac{\partial z_c}{\partial \theta_{ij}}. \label{eqn:lc_1}
\end{equation}
Computing the terms of \eqref{eqn:lc_1}, we have
$\frac{\partial \mathcal{L}}{\partial z_c} = p_t - 1 \mbox{ for } c=t \mbox {; and } \frac{\partial \mathcal{L}}{\partial z_c} = p_c \mbox{ otherwise; }$ 
and $\frac{\partial z_c}{\partial \theta_{{ij}}} = x_j$. Thus, 
\begin{eqnarray}
g(x, \theta)_{{ij}} & = & x_j (p_t - 1) \;\; \mbox{ for i = t} \nonumber \\ 
& = & x_j p_i  \;\; \mbox{ for } i \neq t \label{eqn:lc_2}
\end{eqnarray}
We compute the Hessian of $g(x, \theta)_{ij}$ as:
\begin{eqnarray}
\frac{\partial g(x, \theta)_{ij}}{\partial \theta_{kl}} & = &  x_j p_t (1-p_t) x_l \;\; \mbox{ for } k = t \nonumber \\
& = & x_j p_k (1 - p_k) x_l \;\; \mbox{ for } k \neq t \label{eqn:lc_3}
\end{eqnarray}
where $k \in [C], l \in [m]$. The maximum value of the Hessian in \eqref{eqn:lc_3}, occurs at $x_j = x_l = 1$, and $p_t = p_k = \frac{1}{2}$. Thus, 
\begin{eqnarray}
\max_{i, j, k, l}\frac{\partial g(x, \theta)_{{ij}}}{\partial \theta_{{kl}}} & \leq &  \frac{1}{4} \;\; \mbox{ for } k = t \nonumber \\
& \leq & \frac{1}{4} \;\; \mbox{ for } k \neq t \label{eqn:lc_4}
\end{eqnarray}
To obtain the Lipschitz constant, we first define the function
\[ h(t) = g((1-t)\theta_1 + t \theta_2, x)_{ij} \mbox{ where } t \in [0, 1] \]
Thus, $h(0) = g(\theta_1, x)_{ij}$ and $h(1) = g(\theta_2, x)_{ij}$. Since, the function $h(t)$ is differentiable everywhere in $(0, 1)$, using Mean Value Theorem~\cite{nla.cat-vn298857}, we know that there exists a point $t^{*} \in (0, 1)$ such that:
\begin{equation}
 h(1) - h(0) \leq h'(t^{*}) 
\mbox{ where } h'(t) = (\theta_2 - \theta_1) g'((1-t) \theta_1 + t \theta_2, x)_{ijkl}. \label{eqn:lc_5}
\end{equation}
Rewriting \eqref{eqn:lc_0}, we get
\begin{eqnarray}
& & \sup_{x \in D, \theta_1, \theta_2} |g(\theta_1, x) - g( \theta_2, x) |_1 \nonumber \\
& \leq & \sup_{x \in D, \theta_1, \theta_2} |\max_{i, j} \{ g( \theta_1, x)_{ij} - g(\theta_2, x)_{ij} \} |_1 \nonumber 
\end{eqnarray}
Let $i^{*}, j^{*}$ correspond to the indices where the maximum in the above equation occurs. Combining \eqref{eqn:lc_4} and \eqref{eqn:lc_5}, we get:
\begin{equation}
    \sup_{x \in D, \theta_1, \theta_2} | g(\theta_1, x)_{i^{*} j^{*}} - g(\theta_2, x)_{i^{*} j^{*}} |_1 \leq \frac{1}{4} |\theta_1 - \theta_2|_1 \label{eqn:lc_6}
\end{equation}
Comparing \eqref{eqn:lc_6} with \eqref{eqn:lc_0} we get $c = \frac{1}{4}$.
\subsubsection{Computing the certified radius} \label{appendix:computingradius}

Algorithm \ref{alg:radius} calculates the maximum distance between the poisoned and benign models, based on the number of attackers, protocol parameters $c,\lambda$ defined in Definition \ref{def:clipschitz}, number of iterations $T$, clipping parameter $L$, the dimension of the model $d$ and sparsification parameter $k$.
The correctness of this procedure follows from the proof of Theorem \ref{theorem:main}.
% Then, Algorithm \ref{alg:certify} takes this radius together with a dataset and the Lipschitz constant of the model space $\Theta$ and outputs the certified accuracy.

%   \begin{minipage}[t]{0.5\textwidth}
\begin{algorithm}[H]
\caption{Radius calculation}
\label{alg:radius}
    \begin{algorithmic}
    {\begin{small}
	\REQUIRE poisoning parameter $\rho$, number of model weights to update each round $k$, number of timesteps $T$, decay function $\lambda$, model parameters $\theta$, test dataset $(x,y)_{i=1}^{m}$, Lipschitzness $c$, error $\gamma$
    \STATE $r = 0$
    \STATE $\beta = \epsilon+ \gamma$
	\FOR{$t = 1,2,\cdots T$}
	\STATE $\alpha = 1 + 2\lambda(t) ck$
	\STATE $r = r * \alpha + \lambda(t)\beta$
    \ENDFOR
    \ENSURE radius $r$
	\end{small}}
	\end{algorithmic}
\end{algorithm}
\section{Methods and Metrics}

\subsection{Methods}\label{appendix:methods}
In this section we give a detailed treatment of the methods we compare.
We run all experiments on commercially available NVIDIA Pascal GPUs.
With this in mind, all implementations are optimized to run on a single GPU and all our experiments can be reproduced within a few hours (\algoname{}) or days (Byzantine-robust aggregation methods).

\subsubsection{FedAvg}\label{appendix:fedavg}
We use the standard implementation of federated averaging \cite{mcmahan17fedavg}, described in Algorithm \ref{alg:fedavg}, as the baseline for all defenses in this work.
The first major departure is the use of $\ell_2$ clipping, which is in place whenever we refer to the $\ell_2$ clipping defense.
When we refer to an "undefended" system, we do not make use of $\ell_2$ clipping.
As an implementation detail, we average updates and not individual models because we employ norm clipping in all defenses and clipping model parameters wholesale is more difficult than clipping updates.
The second major departure is the use of server-side momentum, which has empirically been shown to improve convergence \citep{rothchild2020fetchsgd}.

\noindent \textbf{Local epochs make outlier detection difficult:} 
From an adversarial perspective, \fedavg{} has a key vulnerability: the use of multiple local epochs $\tau$, which is a design choice to amortize communication costs.
As the number of local epochs $\tau \rightarrow \infty$, individual updates from benign devices become further apart in $\ell_2$ space.
This makes it difficult for Byzantine-robust aggregation rules such as Bulyan and Krum to identify outliers, because both attacker updates and benign updates are very far apart. 
Therefore, when benign devices do multiple local epochs, attackers are more likely to remain undetected by outlier detection methods.
To ensure we are comparing against the strongest versions of the Byzantine-robust aggregation rules possible, we use $\tau=1$.

\noindent \textbf{Local epochs amplify existing vulnerabilities:} Even when the number of local epochs $\tau=1$, \fedavg{} with $\ell_2$ clipping does not reduce to distributed SGD because devices scale their updates by the learning rate before doing norm clipping.
This presents an opportunity for the attacker: when the global learning rate is very small, such as towards the end of training when using a typical decaying learning rate schedule, the updates of most benign devices will have $\ell_2$ norm close to 0.
Here, the attacker can simply project their update to the perimeter of the $\ell_2$ norm constraint and essentially have an update which is hundreds of times larger than the rest of the benign devices, which enables them to perform model replacement.
In Appendix \ref{appendix:normclipping} we propose and evaluate a method to mitigate this vulnerability.

\noindent \textbf{Model replacement: }Model replacement has already been proposed as an attack strategy in prior work \citep{bagdasaryan18backdoor} because state of the art models often converge to a stationary point towards the end of training.
This vulnerability is simply amplified in federated learning, because all federated learning deployments today make use of multiple local epochs, as update communication is the system bottleneck.

\begin{algorithm}
\caption{\algoname{}}
\label{alg:fedavg}
    \begin{algorithmic}
    {\begin{small}
	\REQUIRE learning rate $\lambda$, number of timesteps $T$, local batch size $b$, number of devices selected per round $n$, norm clipping parameter $L$, local epochs $\tau$, local learning rate $\gamma$
    \STATE Initialize model $\theta_0$ using the same random seed on the devices and aggregator 
    \STATE Initialize momentum vector $\R^t=0$
	\FOR{$t = 1,2,\cdots T$}
	    \STATE Randomly select $n$ devices $d_1,\ldots d_n$
	    \LOOP[In parallel on devices $\bc{d_i}_{i=1}^n$]
	        \STATE Download new model weights $\theta_t = \theta$ 
	        \FOR{$m \in \tau$}
        		\STATE Compute gradient $\g_t^i = \frac{1}{b}\sum_{j=1}^l\nabla_\theta \mathcal{L}(\theta^t, \D_j)$
        		\STATE Accumulate gradient $\theta_t = \theta_t - \gamma(t, m) \g_t^i$ 
    		\ENDFOR
    		\STATE Compute update $\u_t^i = \theta_t - \theta$
    		\STATE Clip update $\u_t^i = \u_t^i \cdots \min(1, \frac{L}{|\u^t_i|_2})$
    	\ENDLOOP
		\STATE Aggregate gradients $\u_t = \frac{1}{n}\sum_{i=1}^n \u_t^i$ 
		\STATE Momentum $\R^t = 0.9 \R^t + \u_t$
	    \STATE Update $ \theta_{t+1} = \theta_{t} - \lambda(t) \R_t$ 
    \ENDFOR 
	\ENSURE  $\bc{\theta^t}_{t=1}^T$
	\end{small}}
	\end{algorithmic}
\end{algorithm}

\noindent \textbf{Uncompressed FL is more robust than \fedavg{}:} 
In Table \ref{table:table-uncompressed} we show that using distributed SGD as the backbone algorithm rather than \fedavg{} has a marked impact on the attack accuracy.
We refer to this regime as "uncompressed FL" because we are not compressing communication costs, and note that this regime is strictly unrealistic.
Even in the uncompressed regime, the attack still functions via model replacement, because the benign objective reaches a stationary point and the gradients from benign devices are very small.
We note that while the attack does not reach $100 \%$ accuracy against the $\ell_2$ defense in this setting, when we incorporate minor adjustments to the attack (Appendix \ref{appendix:attacktuning}) we can reach $100 \%$ accuracy; \algoname{} still functions well as a defense.

In Appendix \ref{appendix:fetchsgd} we introduce a communication-efficient variant of \algoname{} which can drop the use of multiple local epochs altogether, and therefore obtains improved robustness empirically.

\begin{table}
% \centering
\caption{Attack accuracy decrease for $\ell_2$ norm clipping and \algoname{} when doing uncompressed FL (SGD) \emph{as compared to using \fedavg{}}.
CIFAR10, 1e4 clients, 200 attackers.}
\label{table:table-uncompressed}

\centering
\begin{tabular}{llll}
\toprule
\textbf{Defense} & \textbf{Test acc}& \textbf{Attack acc (decrease)} & \textbf{Attack acc}\\
 \midrule
$\ell_2$ & 84.07 $\pm 0.7$ & 34.0 $\pm 6$ & 66.0 $\pm 6$ \\
\algoname{} & 81.72 $\pm 0.9$ & 20.0 $\pm 5$ & 5.6 $\pm 1$ \\
\bottomrule
\end{tabular}
\end{table}

\noindent \textbf{Momentum is necessary for convergence:}
As an implementation detail, we employ momentum factor masking \citep{rothchild2020fetchsgd} in \algoname{}.
This entails maintaining a momentum buffer which we zero out similar to the error feedback vector.
We provide the momentum enabled algorithm in Algorithm \ref{alg:momentumsparsefed}
We do not analyze the role of momentum in robustness, but it is empirically useful.

% A potential sidechannel for the attacker is the momentum buffer.
% Momentum is generally necessary for the convergence of high quality models, but if an attacker can exploit the momentum buffer, they can propagate errors in unused model capacity over iterations.
In Table \ref{table:table-momentum} we see that without the use of momentum, neither the model nor the attack converge when using just \fedavg{} with $\ell_2$ clipping.
This is what \algoname{} reduces to as $k \rightarrow d$, because at every iteration we zero out the entire momentum buffer.

\begin{algorithm}
\caption{\algoname{}}
\label{alg:momentumsparsefed}
    \begin{algorithmic}
    {\begin{small}
	\REQUIRE number of coordinates to update each round $k$, learning rate $\lambda$, number of timesteps $T$, local batch size $b$, number of devices selected per round $n$, norm clipping parameter $L$, local epochs $\tau$, local learning rate $\gamma$
    \STATE Initialize model $\theta_0$ using the same random seed on the devices and aggregator 
    \STATE Initialize memory vector $\W_t=0$, momentum vector $\R^t=0$
	\FOR{$t = 1,2,\cdots T$}
	    \STATE Randomly select $n$ devices $d_1,\ldots d_n$
	    \LOOP[In parallel on devices $\bc{d_i}_{i=1}^n$]
	        \STATE Download new model weights $\theta_t = \theta$ 
	        \FOR{$m \in \tau$}
        		\STATE Compute gradient $\g_t^i = \frac{1}{b}\sum_{j=1}^l\nabla_\theta \mathcal{L}(\theta^t, \D_j)$
        		\STATE Accumulate gradient $\theta_t = \theta_t - \gamma(t, m) \g_t^i$ 
    		\ENDFOR
    		\STATE Compute update $\u_t^i = \theta_t - \theta$
    		\STATE Clip update $\u_t^i = \u_t^i \cdots \min(1, \frac{L}{|\u^t_i|_2})$
    % 		(can be verified by MPC, e.g. range proof)
    	\ENDLOOP
		\STATE Aggregate gradients $\u_t = \frac{1}{n}\sum_{i=1}^n \u_t^i$ 
		\STATE Momentum: $\R^t=0.9 \cdot R^{t-1}+u^t$
 		\STATE Error feedback: $\W_t = \R_t + \W_t$
 		\STATE Extract $top_k$: $\Delta_t = top_k(\W_t)$
		\STATE Error accumulation: $\W_{t+1} = \W_t -\Delta_t$
	    \STATE Update $ \theta_{t+1} = \theta_{t} - \lambda(t) \Delta_t$ 
    \ENDFOR 
	\ENSURE  $\bc{\theta^t}_{t=1}^T$
	\end{small}}
	\end{algorithmic}
\end{algorithm}

\begin{table}
\caption{Test/Attack accuracy decrease for $\ell_2$ norm clipping  when not using momentum.
CIFAR10, 1e4 clients, 200 attackers.}
\label{table:table-momentum}

\centering
\begin{tabular}{lllll}
\toprule
\textbf{Defense} & \textbf{Test Acc (decrease)} & \textbf{Test acc}& \textbf{Attack acc (decrease)} & \textbf{Attack acc}\\
 \midrule
$\ell_2$ & 31.08 $\pm 0.7$ & 53.14 $\pm 1.7$ & 61.4 $\pm 6$ & 4.6 $\pm 1$\\
% \algoname{} & 0.8172 & 20.0 & 5.6 \\
\bottomrule
\end{tabular}
\end{table}

% add these in if i have time
% \noindent \textbf{Trimmed Mean: }
% \noindent \textbf{Krum: }
% \noindent \textbf{Bulyan: }

\subsubsection{The Attack}\label{appendix:attack}
In Algorithm \ref{alg:attack} we provide the model poisoning attack that we use throughout this work.
This attack is similar to the PGD attack proposed in prior work \citep{sun2019backdoor}, with the addition of the attacker batch size parameter which enables us to poison models with larger auxiliary datasets.
In Appendix \ref{appendix:attacktuning} we provide detailed analysis on how we choose the attacker batch size and number of PGD epochs.
The attackers sample data from the "auxiliary dataset", a dataset which is composed of datapoints with their labels flipped that the attacker uses as a proxy to formulate the poisoned gradient.

\begin{algorithm}
\caption{Attack}
\label{alg:attack}
    \begin{algorithmic}[1]
    {\begin{small}
	\REQUIRE learning rate $\eta$, local batch size $\ell$, norm clipping parameter $L$, number of local epochs $e$
    \STATE This procedure is used by all attackers in a round to ensure that they upload the same update
    \FOR{number of PGD epochs $e_i \in e$}
        \STATE Compute stochastic gradient $\g_i^t$ on batch $B_i$ of size $\ell$: $\g^t_i = \frac{1}{\ell}\sum_{j=1}^l\nabla_\M \mathcal{L}(\M^t_{e_i}, \D_j)$
    	\STATE Update local model $\hat{M}^t_{e_{i+1}} = \M^t_{e_i} - \eta \g_i^t$
    	\STATE Project accumulated update onto the perimeter of the $\ell_2$ constraint
    	$\M^t_{e_{i+1}} = M^t_0 - CLIP(\hat{M}^t_{e_{i+1}} - M^t_0)$
    \ENDFOR 
	\ENSURE  $\M^t_{e}$
	\end{small}}
	\end{algorithmic}
\end{algorithm}

\subsubsection{Byzantine-resilient defenses}\label{appendix:byzantinedefenses}
Every algorithm we describe in this section is implemented via replacing line 15 in Algorithm \ref{alg:fedavg}.
This introduces additional computational complexity into the aggregation step, which is the bottleneck in federated learning.
This complexity can be minor (trimmed mean) or it can be massive (Bulyan).
Our experiments with Bulyan take approximately $20 \times$ longer to run than our experiments with \algoname{}; because these experiments are so computationally infeasible, where possible we omit Bulyan from comparisons in the rest of the Appendix.
These defenses as initially proposed do not make use of $\ell_2$ norm clipping, but because we use $\ell_2$ clipping in the baseline defense, and because it benefits all defenses (Appendix \ref{appendix:normclipping}), the input gradients to all the aggregation rules are already clipped.

\begin{algorithm}
\caption{Trimmed mean}
\label{alg:trimmedmean}
    \begin{algorithmic}[1]
    {\begin{small}
	\REQUIRE number of compromised devices $f$, set of individual updates $U=\bc{u^t}_{i=1}^{n}$
	\FOR{number of compromised devices $f$}
	    \FOR{each coordinate $\bc{c}_{j=1}^{d}$}
	        \STATE $U_c \gets U_c \setminus \min U_c$
	        \STATE $U_c \gets U_c \setminus \max U_c$
        \ENDFOR
	\ENDFOR
	\STATE Aggregate remaining updates $\u^t = \frac{1}{n-2f}\sum_{i=1}^{n-2f} \u^t_i$ 
	\ENSURE  $\u^t$
	\end{small}}
	\end{algorithmic}
\end{algorithm}
\noindent \textbf{Trimmed mean:}
In Algorithm \ref{alg:trimmedmean} we see that trimmed mean iteratively rejects outliers at each coordinate until it has eliminated $2f$ coordinates.
If the attacker's updates have extremely small or large values, then trimmed mean will mitigate the attack.
However, if most of the attacker's updates are close to $0$ at many coordinates, then trimmed mean will not mitigate the attack.
This is the phenomena observed in \citep{pmlr-v97-bhagoji19a}; the attacker's updates are far sparser than benign updates, which in turn means that most coordinate values are $0$ and thus trimmed mean is ineffective.

\noindent \textbf{Coordinate median:}
Coordinate median is simply implemented by returning the coordinatewise median instead of the mean.
This does not converge because of the gap between median and mean \citep{Chen2020DistributedTW, MuozGonzlez2019ByzantineRobustFM, yin2018byzantine}.

\begin{algorithm}
\caption{Krum}
\label{alg:krum}
    \begin{algorithmic}[1]
    {\begin{small}
	\REQUIRE number of compromised devices $f$, set of individual updates $U=\bc{u^t}_{i=1}^{n}$
	\FOR{each update $u^t_i$}
	    \STATE $U_i = U$
	    \FOR{f+2}
	        \STATE $U_i = U_i \setminus \argmax_{u^t_j \in U_i} \norm{u^t_j - u^t_i}$
	    \ENDFOR
	    \STATE $S_i = \sum_{u_j \in U_i} \norm{u^t_j - u^t_i}$
	\ENDFOR
	\STATE 
	\ENSURE  $\u^t = \argmin_{u \in U} S$
	\end{small}}
	\end{algorithmic}
\end{algorithm}
\noindent \textbf{Krum:}
Algorithm \ref{alg:krum} implements Krum, which attempts a Byzantine-resilient variant of the barycentric aggregation rule \citep{blanchard2017machine}.
Krum selects a single update from the aggregated set to update the global model.
In the cross-device federated setting, this will never converge.
Essentially, we will be using SGD instead of minibatch SGD, and it will take us $100 \times$ longer to do one pass over the entire dataset.
Because Bulyan uses Krum and trimmed mean, we do not analyze Krum in isolation in depth.

\begin{algorithm}
\caption{Bulyan}
\label{alg:bulyan}
    \begin{algorithmic}[1]
    {\begin{small}
	\REQUIRE number of compromised devices $f$, set of individual updates $U=\bc{u^t}_{i=1}^{n}$
	\STATE $\Theta = n - 2f$
	\STATE $S = \emptyset$
	\WHILE{$|S| < \Theta$}
	    \STATE $p = \texttt{KRUM}(U,f)$
	    \STATE $U \gets U \setminus p$
	    \STATE $S \gets S \cup p$
	\ENDWHILE
	\ENSURE  $\u^t = \texttt{TRIMMEAN}(S,f)$
	\end{small}}
	\end{algorithmic}
\end{algorithm}

\noindent \textbf{Bulyan:}
Algorithm \ref{alg:bulyan} describes Bulyan \citep{mhamdi2018hidden} implemented with Krum as the base aggregation rule.
Bulyan builds a set by iteratively applying Krum onto the set of aggregated updates, and then returns the trimmed mean of this set.
If Krum selects the attacker, we already know that trimmed mean is not likely to reject the attacker.
Thus, it remains to intuit why Krum will select at least one attacker.
In the non-i.i.d.\ setting, benign update vectors are sufficiently far away that a very small number of colluding attackers at each iteration can minimize their distance to all other vectors by sending the same update, which ensures that they have a distance of $0$ from each other.
Thus, Krum selects at least one attacker, and Bulyan fails, as we show in our experiments.

It is readily apparent that for large values of $n$, Bulyan is fairly computationally inefficient even when implemented efficiently.
Although the asymptotic complexity of Bulyan is the same as that of Krum, the constant factor is quite large ($n=100$).

\subsubsection{\algoname{}}\label{appendix:sparsefed}
In the main body we include the algorithm for \algoname{} implemented with true top-$k$.
As an implementation detail, the algorithm presented in the main body is in the uncompressed regime, where we do not perform any local epochs and the learning rate is multiplied after the top-$k$ coordinates are extracted.

\begin{algorithm}
\caption{\algoname{} implemented with \fedssgd{} instead of global top-$k$}
\label{alg:fetchsgd}
    \begin{algorithmic}[1]
    {\begin{small}
	\REQUIRE number of model weights to update each round $k$
	\REQUIRE learning rate $\eta$
	\REQUIRE norm clipping parameter $L$
	\REQUIRE number of timesteps $T$
	\REQUIRE momentum parameter $\rho$, local batch size $\ell$
	\REQUIRE Number of clients selected per round $W$
	\REQUIRE Sketching and unsketching functions $\mathcal{S}$, $\mathcal{U}$
    \STATE Initialize $\bS_u^0$ and $\bS_e^0$ to zero sketches
    \STATE Initialize model $\theta_0$ using the same random seed on the devices and aggregator 
	\FOR{$t = 1,2,\cdots T$}
	    \STATE Randomly select $n$ devices $d_1,\ldots d_n$
	    \LOOP[In parallel on devices $\bc{d_i}_{i=1}^n$]
	        \STATE Download new model weights $\theta_t = \theta$
    		\STATE Compute gradient $\g_t^i = \frac{1}{b}\sum_{j=1}^l\nabla_\theta \mathcal{L}(\theta^t, \D_j)$
    		\STATE Clip $\g^t_i$ according to $L$: $\g^t_i = \g^t_i * \min(1, \frac{L}{|\g^t_i|_2})$
    		\STATE Sketch $\g_i^t$: $\bS_i^t = \mathcal{S}(\g_i^t)$ and send it to the Aggregator 
    	\ENDLOOP
		\STATE Aggregate sketches $\bS^t = \frac{1}{W}\sum_{i=1}^W \bS_i^t$ 
		\STATE Momentum: $\bS_u^t = \rho \bS_u^{t-1} + \bS^t$ 
 		\STATE Error feedback: $\bS_e^t = \eta\bS_u^t + \bS_e^t$ 
 		\STATE Unsketch: $\Delta^t = \text{Top-k}(\mathcal{U}(\bS_e^t))$
		\STATE Error accumulation: $\bS_e^{t+1} = \bS_e^t -S(\Delta^t)$
	    \STATE Update $ \theta^{t+1} = \theta^{t} - \Delta^t$ 
    \ENDFOR 
	\ENSURE  $\bc{\w^t}_{t=1}^T$
	\end{small}}
	\end{algorithmic}
\end{algorithm}
\noindent \textbf{\fedssgd{}:}
Algorithm \ref{alg:fetchsgd} is the \fedssgd{} algorithm \citep{rothchild2020fetchsgd} combined with $\ell_2$ clipping.
\fedssgd{} approximates true top-$k$ and has been empirically shown to be communication efficient; in Appendix \ref{appendix:fetchsgd} we validate the robustness of \algoname{} implemented with \fedssgd{}.
Because \algoname{} implemented with \fedssgd{} can achieve communication efficiency without the use of multiple local epochs, it has improved robustness over \algoname{} implemented with true top-$k$, which still requires multiple local epochs for communication efficiency.

\subsubsection{Adaptively choosing k in \algoname{}}\label{appendix:adaptivekselection}
The hyperparameter $k$ is critical for the convergence of \algoname{}. 
In Algorithm \ref{alg:adaptiveselectionk} we provide an adaptive algorithm for selecting $k$.
The algorithm requires as input the maximum  information loss tolerance due to sparsification, and essentially just performs binary search over a range of reasonable values of $k$ until finding the smallest $k$ that does not lose "too much" information.

\begin{algorithm}
\caption{Selecting $k$}
\label{alg:adaptiveselectionk}
    \begin{algorithmic}[1]
    {\begin{small}
	\REQUIRE model $\theta$, maximum information loss $\omega$, number of model parameters $d$, number of iterations in an epoch $r$, number of gradients to sample $n$ (more samples gives a better estimate of $\omega$)
	\STATE set initial k $k = \frac{d}{r}$
	\STATE set initial realized information loss $\delta = \infty$
	\WHILE{$\delta > \omega$}
	    \STATE compute $n$ sample minibatch gradients $\{g\}_{j=1}^{n} | g_j = \nabla_\theta \mathcal{L}(\theta, \z_j)$
	    \STATE extract top-$k$ $\{u\}_{j=0}^n | u_j = top_k(g_j)$
	    \STATE calculate average $l_1$ mass lost $\delta^{*} = \frac{1}{n} \sum_{j=1}^{n} |g_j - u_j|_{1}$
	    \STATE update $\delta = \min(\delta, \delta^{*})$
	    \IF{$\delta > \omega$}
	        \STATE $k = k + \frac{d}{r}$
	    \ENDIF
	\ENDWHILE
	\ENSURE $k$
	\end{small}}
	\end{algorithmic}
\end{algorithm}

\subsubsection{Metrics}\label{appendix:oif}
In the main body, we mainly use the attack accuracy metric for the fixed cross-silo and cross-device settings.
However, in the rest of the Appendix we do not always use this setting when it does not illustrate the full breadth of a trend, and we note that attack accuracy is not a perfect metric.
For example, when trying to poison 1 point, the attackers can trivially obtain $100 \%$ attack accuracy, but this is not the case when they are trying to poison 100 points.
Similarly, $100$ attackers will have an easier time poisoning 1 point than $1$ attacker will.
To address these shortcomings, we introduce a new metric.

% add an example for why attack accuracy is not a good metric
% e.g. misclassifying 1 point vs misclassifying 100 points
\noindent \textbf{Outsized Impact Factor (OIF)} 
We first define some notation. 
Let $S$ be the set of agents participating in federated learning, and $S_b$ the set of benign agents so that $I = \frac{|S\setminus S_b|}{|S|}$ is the influence of the attacker on the system, represented as the fraction of agents which are compromised. 
We propose that the baseline for any model poisoning attack should be for the attackers to be able to poison datapoints (e.g. flip the label on that datapoint) $\hat{X}_m$ proportional to their influence $I$. 
Therefore, if $|\hat{X}_m|$ is the number of datapoints successfully poisoned and $n$ is the total number of datapoints controlled by all agents in the system,
we define $\frac{|\hat{X}_m|}{I \cdot n}$, which is the ratio of datapoints successfully poisoned relative to the influence of the attacker, normalized by the size of the dataset, as the \emph{outsized impact factor} (OIF). 
This quantity determines the extent to which the attacker is able to `punch above its weight' in terms of impacting the final model to a larger extent than its influence would already allow.

Our standard for a successful attack is an OIF of 1.
This means that the attacker can poison the same fraction of the dataset as of the client population they control. 
By using the OIF metric as a heuristic for attack success, we can easily compare the efficacy of attacks across parameter settings when different numbers of attackers are present.
\subsection{Norm Clipping}\label{appendix:normclipping}

\noindent \textbf{Adaptive clipping to mitigate the vulnerability of \fedavg{}}
As we note in Appendix \ref{appendix:fedavg}, the key vulnerability of \fedavg{} is that benign devices multiply their gradients by a small learning rate that can vary over the course of training, which can make their gradients smaller than the specified $\ell_2$ norm clipping bound when the learning rate is small (e.g. when warming up the learning rate schedule at the start of training).
However, the attack is under no such compulsion, and this can present an easy vulnerability for the attacker.
To mitigate this, we propose the use of an adaptive $\ell_2$ clipping schedule which simply mirrors the learning rate schedule.
At each iteration, before we clip the device gradient to the specified norm $L$, we scale $L$ by the learning rate $L := L \cdot \lambda(t)$.
In Table \ref{tab:adaclip} we ablate the effectiveness of this on trimmed mean and Bulyan.

\begin{table}
    \centering
    \begin{tabular}{c|cc}
         Defense & Attack Accuracy (without) & Attack Accuracy (with) \\
         Trimmed mean & 100 & 81.4 \\
         Bulyan & 100 & 81.8 \\
    \end{tabular}
    \caption{In the cross-device setting of CIFAR10, trimmed mean and Bulyan benefit greatly from the use of adaptive clipping.}
    \label{tab:adaclip}
\end{table}

\noindent \textbf{Sparsification needs norm clipping}

We perform ablations of the central idea of the paper, sparsification as a defense against model poisoning attacks, with and without the use of $\ell_2$ norm clipping.

In Fig. \ref{fig:clipablationsparse} we compare the efficacy of the combination of the distributed poisoning attack and the PGD attack against the $top_k$ defense, with and without $\ell_2$ clipping with parameter $3$. 
We observe that when $\ell_2$ clipping is in place, sparsification completely mitigates the attack.
However, without any clipping the attacker is able to successfully flip the labels of their entire auxiliary dataset.
This is because without any constraint on the norm of its update, the attacker can massively magnify its update and ensure that all the coordinates in the $top_k$ are in the direction of the adversarial optimum.

% In Fig. \ref{fig:layerfreezingclipping} we compare the baseline attack against the layer freezing defense with and without $\ell_2$ clipping with parameter $5$.
% When no clipping is in place, the attack succeeds completely regardless of how many layers are frozen as long as the entire model is not frozen.
% Freezing the entire model corresponds to early stopping, which we do not evaluate in this work.
% When we use clipping, we observe that the attack is considerably mitigated.
% The more layers we freeze, the better the mitigation and the test accuracy.
% The results of this plot are summarized in the main body in Table 5.
% \ref{table:table-layerfreezing}.

\begin{figure}[t]
\centering
\includegraphics[width=210px]{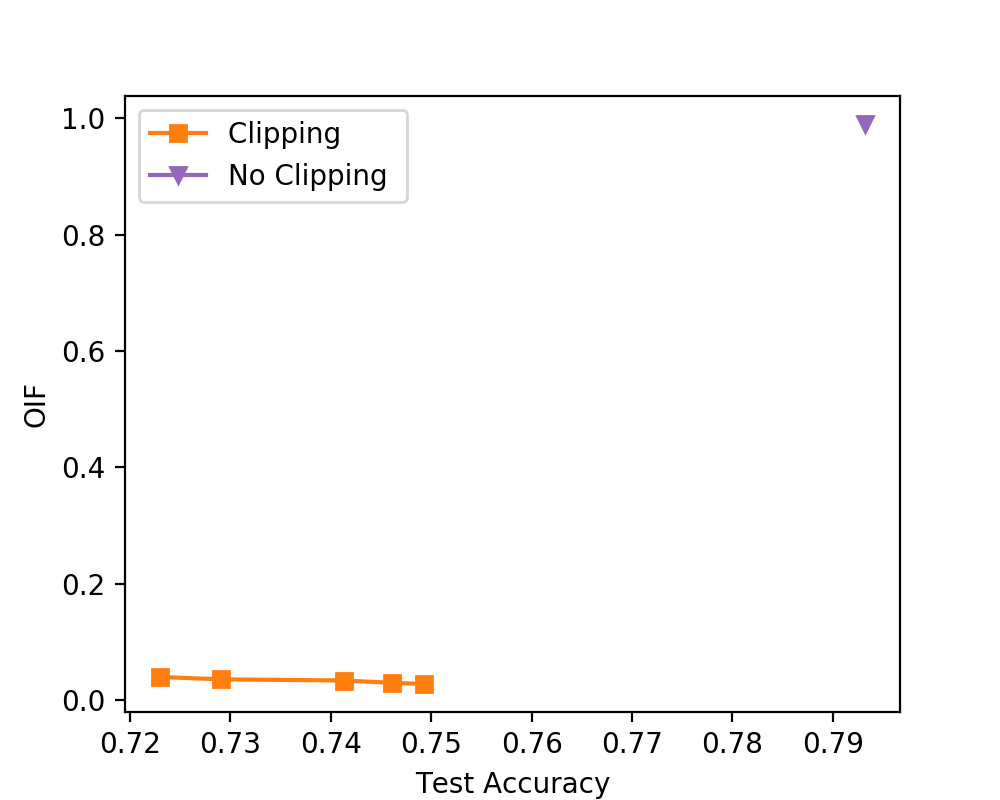}
\caption{
Pareto frontier of the combination of distributed poisoning and PGD attacks against \algoname{} defenses with and without $\ell_2$ clipping.
Without $\ell_2$ clipping, sparsification is entirely unable to mitigate the attack.
CIFAR10, 10000 devices, 100 attackers.
}
\label{fig:clipablationsparse}
\end{figure}

% \begin{figure}[t]
% \centering
% \includegraphics[width=210px]{plots/layerfreezingclipping.png}
% \caption{
% Pareto frontier of the layer freezing defense, with and without $\ell_2$ clipping.
% When clipping is not in place, the layer freezing defense is completely ineffective in mitigating the attack.
% }
% \label{fig:layerfreezingclipping}
% \end{figure}

\noindent \textbf{Byzantine-Robust Aggregation Benefits from Norm Clipping}
The prior defenses we consider (Krum, Bulyan, trimmed mean, coordinate median) do not require norm clipping as part of the implementation.
Norm clipping will either help the defense by limiting the impact of the attacker, in which case the server will enforce norm clipping, or it will hurt the defense by making the attack more stealthy, in which case the attacker will use norm clipping.
In Table \ref{table:table-byzantine-clipping} we compare the changes in test and attack accuracy for Bulyan and trimmed mean when implementing norm clipping (Krum and coordinate median do not converge).
As expected, norm clipping limits the impact of the attacker and helps Bulyan mitigate the attack when no colluding attackers are present.

\begin{table}
\centering
\caption{Implementing norm clipping greatly mitigates the effectiveness of the attack against Bulyan and trimmed mean when no colluding attackers are present.
CIFAR10, 1e4 devices, 100 attackers.}
\label{table:table-byzantine-clipping}

\centering
\begin{tabular}{lll}
\toprule
\textbf{Defense} & \textbf{Test acc} & \textbf{Attack acc}\\
 \midrule
Bulyan ($\ell_2$) & 83.64 & 10.0 \\
Bulyan & 84.94 & 38.6 \\
Trimmed Mean ($\ell_2$) & 77.42 & 71.6 \\
Trimmed & 81.99 & 100.0 \\
\bottomrule
\end{tabular}
%\vspace*{-20pt}
\end{table}

\subsubsection{Robustness in the DP defense costs accuracy}

Prior work proposed combining $\ell_2$ norm clipping and adding Gaussian noise to ensure robustness, similar to the process adopted in DP-SGD.
In this work, we assume that practitioners will not be willing to adopt defenses which negatively impact the test accuracy of their models in scenarios where attackers are not present.
We note that this is distinct from the accuracy degradation incurred from using a communication-efficient algorithm such as \fedssgd{} as a defense, or deploying DP-SGD to ensure differential privacy.
In these cases, adversarial robustness can be seen as an additional benefit that `comes for free'.
However, the parameters that we find allow for some adversarial robustness at the cost of test accuracy for the DP defense do not actually enable any differential privacy.
As a result, we do not use these parameters for most of our experiments because \emph{we do not believe practitioners will adopt a defense which significantly negatively impacts their model performance.}

In Fig. \ref{fig:noiseclip5} we examine the effect of adding noise $n \sim \mathcal{N}(0, \sigma^2=0.001)$.
This noise parameter is identical to the one chosen in \footnote{Sun et. al. 2019: https://arxiv.org/abs/1911.07963}
% \citep{sun2019backdoor}.
As mentioned above, this amount of noise is entirely insufficient to ensure any differential privacy guarantees.
We show the pareto frontier of the combination of distributed poisoning and PGD against the $\ell_2$ defense with a parameter of $5$, with and without noise addition.
We find that when no attackers are present, adding noise reduces the test accuracy by a minimum of $12 \%$, whereas not adding noise does not reduce the test accuracy at all.
Therefore, while adding noise can make the model more robust, it is also guaranteed to significantly degrade model performance.
In keeping with the aforementioned systemic assumption that practitioners will not use defenses which damage model performance, we do not use noise addition in most experiments.

We nevertheless perform a comparison of the DP defense with $\ell_2$ parameter $5$ and noise addition with $\sigma^2=0.001$, against \algoname{} in Fig. \ref{fig:noisevssparsefed}.
Our findings reinforce our prior conclusions. 
While adding noise with strict clipping is sufficient to mostly mitigate the attack, it comes at the cost of an egregious $20 \%$ test accuracy drop.
By comparison, \algoname{} suffers little accuracy degradation and mitigates the attack even better.

\begin{figure}[t]
\centering
\includegraphics[width=210px]{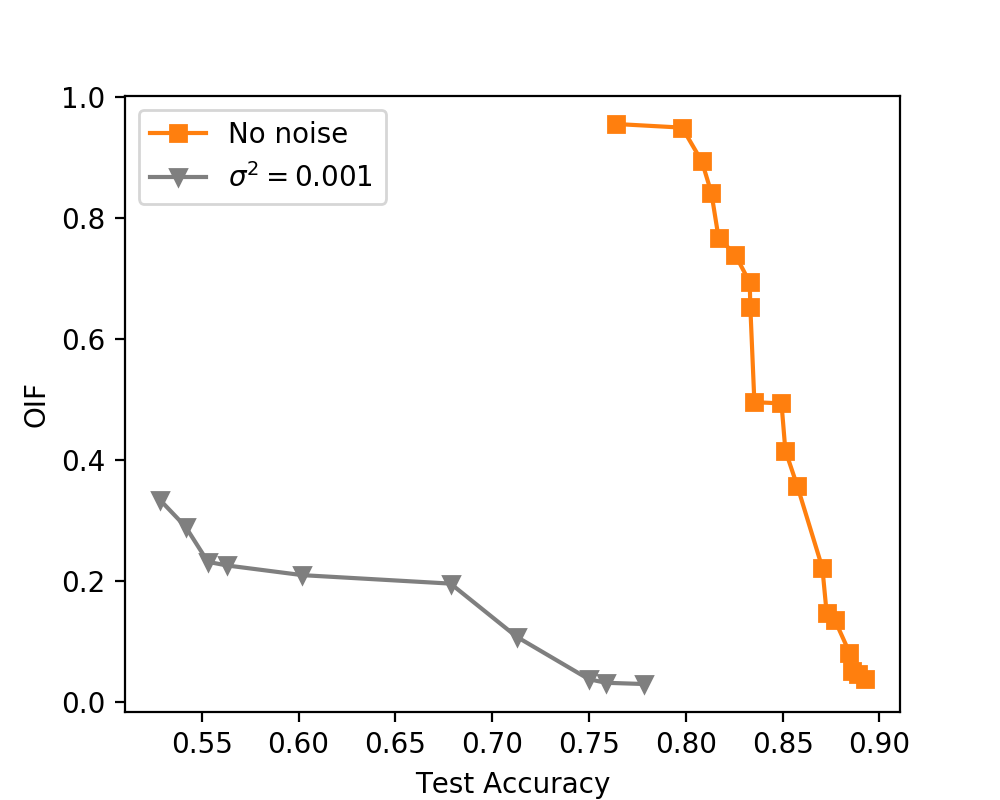}
\caption{Pareto frontier of the $\ell_2$ defense with clipping parameter 5, with and without noise addition, against the attack.
Although noise addition can improve the robustness of the model to attackers, it also degrades test accuracy.
In situations where no attackers are present, adding enough noise to mitigate any possible attackers will reduce the test accuracy by $>10\%$.
Because we do not expect practitioners will adopt any defense which is guaranteed to reduce the performance of their models by such a nontrivial amount, we do not use noise addition.
(points with low OIF either do not make use of PGD or have too small batch sizes)
CIFAR10, 10000 devices, 100 attackers.
}
\label{fig:noiseclip5}
\end{figure}

\begin{figure}[t]
\centering
\includegraphics[width=210px]{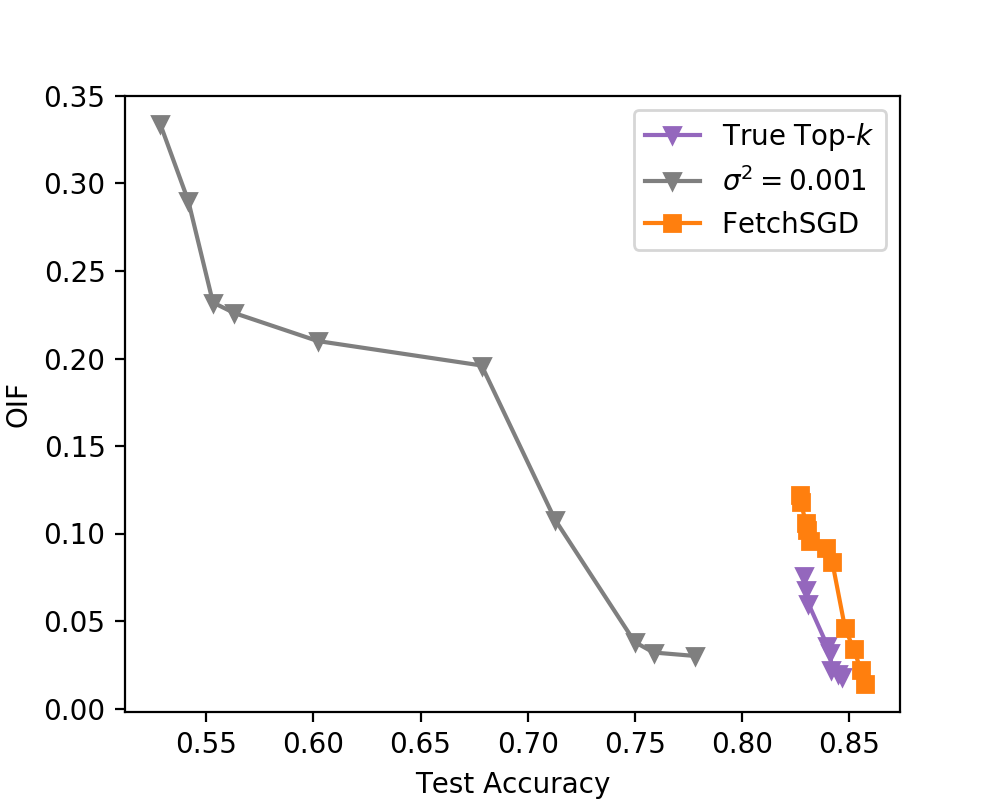}
\caption{
Pareto frontier of the $\ell_2$ defense with noise addition and clip parameter 3, and \algoname{} implemented with $top_k$ and \fedssgd{} with clip parameter 3, against the combination of distributed poisoning and PGD.
The attack was given the same grid search against all 3 defenses: $[50,100,200,400] \times [5,7,9]$. 
Although noise addition is able to mitigate the attack, it suffers dramatically reduced test accuracy when compared to \algoname{}; \algoname{} achieves lower OIF with $10 \%$ higher test accuracy.
CIFAR10, 10000 devices, 100 attackers.
}
\label{fig:noisevssparsefed}
\end{figure}

% \subsection{Additional Experiments}\label{appendix:experiments}
\subsection{Hyperparameter Tuning}
\label{appendix:hyperparametertuning}

\subsubsection{Dataset Parameters}
\noindent \textbf{CIFAR Parameters:}
\label{appendix:cifar}
In all experiments we train for 24 epochs, with $1\%$ of clients participating each round, for 2400 total iterations.
We use the standard train/test split of 50000/10000.
We split the dataset into 10000 clients, each of which has 5 points from a single target class.
In each round we have 100 clients participating, inducing a batch size of 500 (this is of course increased when an adversary participates).
We use standard data augmentation techniques: random crops, random horizontal flips, and the images are normalized according to the mean and standard deviation during training and testing. 
We do not use batch normalization in any of our experiments, because batch normalization does not work well on batches of 5 (batch normalization has to be conducted at a per-client level).
We use a triangular learning rate schedule which peaks at $0.2$.
We use a momentum constant of $0.9$.
These training procedures and the ResNet9 architecture are drawn from Page
\footnote{https://myrtle.ai/learn/how-to-train-your-resnet/}.
% \citep{davidpage}.

\noindent \textbf{FEMNIST Parameters:}
\label{appendix:femnist}
The FEMNIST dataset is composed of 805,263 $28\times28$ pixel grayscale images which are distributed unevenly across 3,550 classes/users. Per user, there are an average of 226.83 datapoints, with a standard deviation of 88.94. 
To preprocess the data, we use the script in the LEAF repository with the command: \texttt{./preprocess.sh~-s~niid~--sf~1.0~-k~0~-t~sample}.
After discarding some datapoints, we end with a dataset of 706,057 training samples and 80,182 validation samples across 3,500 clients ala Leaf
\footnote{https://tinyurl.com/u2w3twe}.

The model architecture we use is a 40M-parameter ResNet101, but we replace the batch norm with layer norm because batch norm does not work well with small batch sizes. 
The average batch size is $\approx 600$ but it can vary based on the clients that are sampled.
We again use the standard data augmentations of random cropping and flips, and a triangular learning rate schedule.
We train for only 1 epoch; this mimics the federated setup where we expect to only use each client once.
We increase the learning rate from 0 to 0.01 over $\frac{1}{5}$th of the dataset, and then decrease the learning rate back to 0.

\noindent \textbf{\fedavg{} Parameters:}
As we discuss in Appendix \ref{appendix:methods}, we use a standard implementation of \fedavg{} where there are three algorithmic hyperparameters: the number of local epochs, the local batch size, and the local learning rate decay.
Prior work has already shown that the use of multiple local epochs does not improve convergence in the regime of small and non-i.i.d.\ datasets \citep{rothchild2020fetchsgd}, and multiple algorithmic variants have been proposed to address this \citep{fedprox} which are out of the scope of this work.
Furthermore, the prior defenses considered in this work rely on approximating some consensus mechanism between benign devices based on the closeness or agreement of benign updates \citep{mhamdi2018hidden}.
As the number of local epochs increases, this consensus falls apart, and so for the sake of fairness we do not evaluate defenses with more than one local epoch.
In Table \ref{table:table-fedavg} we do our own experiments to validate that \fedavg{} convergence does not benefit from multiple local epochs.

\begin{table}
\caption{\fedavg{} convergence does not benefit from doing multiple local epochs.
We use local learning rate=0.9, but even for a small number of local epochs convergence does not benefit, and at these small number of local epochs a smaller local learning rate would not have much impact because the exponential decay factor is not large.
CIFAR10, 10000 devices, no attackers.
}
\label{table:table-fedavg}
\centering
\begin{tabular}{lll}
\toprule
\textbf{Num. epochs} & \textbf{Test acc decrease}& \textbf{Test acc}\\
 \midrule
1 & 0 & 90 \\
2 & 0.41 & 89.59 \\
5 & 80 & 10 \\
\bottomrule
\end{tabular}
\end{table}

\subsubsection{Defense parameters}
\noindent \textbf{Norm clipping parameter:}
For the $\ell_2$ defense, we tune the value of the clipping parameter.
We test values for the $\ell_2$ defense of $(1,3,5,10)$.
Where possible, we do a grid search over as many parameters as possible to find the limit of the attacker's ability.

\begin{table}
\caption{The appropriate choice of the norm clipping parameter greatly mitigates the effectiveness of the baseline attack on CIFAR with auxiliary set of size 500.
CIFAR10, 10000 devices, 100 attackers.
}
\label{table:table-ncd-effective}
\centering
\begin{tabular}{lll}
\toprule
\textbf{Clipping param.} & \textbf{Test acc}& \textbf{Attack acc}\\
 \midrule
10 & 0.7972 & 1 \\
\textbf{5} & 0.83 & 0.136 \\
1 & 0.691 & 0.014 \\

\bottomrule
\end{tabular}
\end{table}
First we validate the $\ell_2$ defense against the baseline attack empirically in Table \ref{table:table-ncd-effective}, which shows that by appropriately choosing the $\ell_2$ parameter, the OIF is reduced significantly.
There is a clear tradeoff: using stricter $\ell_2$ norm clipping mitigates the attack further, but at the cost of reduced test accuracy.

In Fig. \ref{fig:3vs5} we examine the effect of using stricter clipping in the $\ell_2$ defense.
We show the pareto frontier of our attack against the $\ell_2$ defense with two choices of the $\ell_2$ parameter: $3$ and $5$.
We find that when no attackers are present, using a parameter of $3$ admits a minimum of $5 \%$ test accuracy degradation, while using a parameter of $5$ does not reduce test accuracy at all in the same scenario.
Therefore, while using a smaller norm clipping parameter can make the model more robust, it is also guaranteed to always reduce test accuracy.
In keeping with the aforementioned systemic assumption that practitioners will not use defenses which damage model performance, we use the parameter of $5$ in most experiments.
For all further experiments, we show $5$ as the parameter for the $\ell_2$ defense, balancing test accuracy and adversarial robustness.

\begin{figure}[t]
\centering
\includegraphics[width=210px]{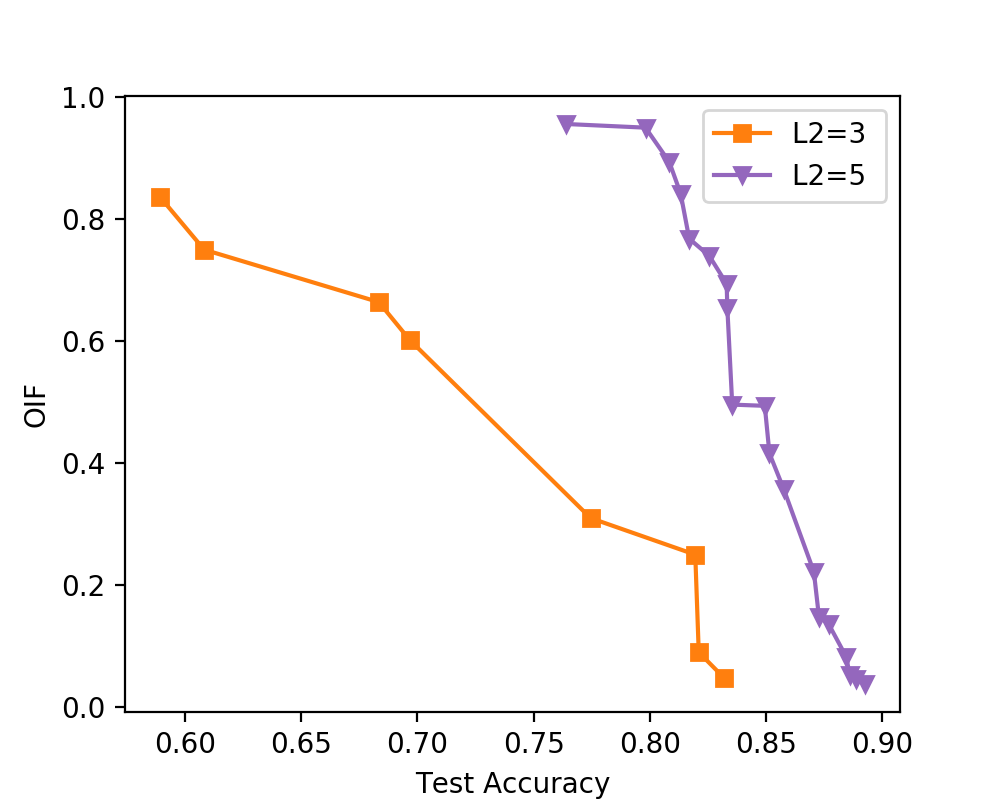}
\caption{Pareto frontier of the $\ell_2$ defense, comparing clipping parameters of 3 and 5.
Although using a stricter norm clipping parameter can reduce OIF, it comes at the cost of test accuracy degradation.
We find that when no attackers are present, using a norm clipping parameter of 5 does not sacrifice any test accuracy, whereas using a norm clipping parameter of 3 sacrifices $>5\%$ test accuracy.
Because we do not expect practitioners will adopt any defense which is guaranteed to reduce the performance of their models by such a nontrivial amount, we use a clipping parameter of 5.
CIFAR10, 10000 devices, 100 attackers.
}
\label{fig:3vs5}
\end{figure}

\noindent \textbf{\algoname{} parameters:}
For \algoname{} we tune the value of $k$, the number of coordinates which are updated at each iteration.
We test values of $[1,5,10,50,100,200,400]\times 10^3$ and report most experiments using the value of $5 \times 10^3$ on CIFAR10/CIFAR100/FMNIST, and use the value of $400 \times 10^3$.
In the main body, we include graphs for the tradeoffs revolving around $k$.
% \arjun{Any conclusions from these variations? Or links to the tables/figures with the corresponding results?}

In Fig. \ref{fig:ktradeofffull} we show the tradeoff between $k$, test accuracy, and attack accuracy for the uncompressed setting.
In the main body, \fedavg{} is the baseline and as noted in \citep{bagdasaryan18backdoor}, the attacker can simply perform model replacement at the last iteration because the learning rate is nearly $0$.
However, in the uncompressed setting this is not possible, so we do not see the same trend as in the main body.
In Appendix \ref{appendix:fetchsgd} we showcase an algorithm which can realistically be implemented without using \fedavg{} to compress communication costs.

\begin{figure}[t]
\centering
\includegraphics[width=210px]{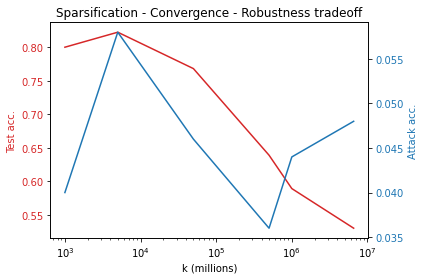}
\caption{
Tradeoff between sparsification parameter $k$ (x axis, in logscale from $1000$ to $k=d=6568640$), test accuracy when attackers are present (left axis, blue), and attack accuracy (right axis, red) for \emph{uncompressed} FL.
In the uncompressed setting, no choice of $k$ allows the attack to succeed, because as $k \rightarrow d$ no momentum is present and neither the attack nor the model converge.
CIFAR10, 10000 devices, 200 attackers.
}
\label{fig:ktradeofffull}
\end{figure}
\subsection{Impact of Defenses on Test Accuracy}
\label{appendix:defenseimpact}

Practitioners in federated learning prioritize the convergence of their models, and attempt to optimize tradeoffs of convergence with communication efficiency, security, and privacy.
In Table \ref{table:table-defense-impact-2} we show the decrease in test accuracy when no attackers are present for each defense evaluated in this work.
We train each model for exactly $2400$ iterations using the same triangular learning rate schedule.
Because the Byzantine-resilient aggregation rules rely on outlier detection, they must necessarily throw away information even when attackers are not present.
We set the robustness parameter $f=5$ to give an idea of the tradeoff for these algorithms, because including a full curve is computationally infeasible.
Bulyan drops more test accuracy than trimmed mean, because Bulyan throws away $4f+2$ updates at each coordinate whereas trimmed mean only throws away $2f$ updates at each coordinate.
As we explain in the main body of the work, Krum and coordinate median do not converge in this setting.

\begin{table}
\caption{Comparing the impact on test accuracy of the defenses.
CIFAR10, 10000 devices, no attackers (averaged over 3 runs).}
\label{table:table-defense-impact-2}
\centering
\begin{tabular}{lll}
\toprule
\textbf{Defense} & \textbf{Test Acc. decrease}& \textbf{Test Acc}\\
 \midrule
No defense & 0 $\pm 0$ & 90.0 $\pm 0.1$ \\
$\ell_2$ & 2.0 $\pm 0.1$ & 88.0 $\pm 0.1$ \\
Krum & 80.0 $\pm 0$ & 10.0 $\pm 0$ \\
Median & 80.0 $\pm 0$ & 10.0 $\pm 0$ \\
Trimmed mean ($f=5$) & 12.58 $\pm 0.8$ & 77.42 $\pm 0.8$\\
Bulyan ($f=5$) & 18.88 $\pm 0.79$ & 71.12 $\pm 0.79$ \\
Bulyan ($f=10$) & 66.48 & 23.52 \\
\algoname{} ($k=5e3$) & 6.82 $\pm 0.7$ & 83.18 $\pm 0.7$ \\
\algoname{} ($k=5e4$) & 3.0 $\pm 0.01$ & 87.0 $\pm 0.01$ \\
\bottomrule
\end{tabular}
\end{table}
% In this work, we assume that practitioners will not be willing to adopt defenses which negatively impact the test accuracy of their models even in scenarios where attackers are not present.

\subsection{Stealth of Attack}
\label{appendix:stealth}
\noindent \textbf{Successful attacks are stealthy attacks:}
A necessary component of a successful attack is relative stealth.
If an attacker can only successfully poison the model by overwriting all of the model's parameters that are necessary to achieve good performance on benign data, we do not consider this a viable attack. 
In any practical deployment, the entity coordinating federated learning would simply discard a model with such low accuracy after running the model on a private test set.
We draw points for the auxiliary dataset from the test set.
This can force the test accuracy to drop by as much as $5 \%$ when the attacker poisons the model with perfect accuracy over an auxiliary set of size $500$ out of a test set of total size $10000$.
In Table \ref{table:table-stealth} we include the decrease in test accuracy on the validation set \textbf{not including the auxiliary set} of size 500
, and confirm that the attack has an element of stealth.
For the attacks on CIFAR10, CIFAR100, and FMNIST, the auxiliary dataset is drawn randomly from all classes and the decrease in test accuracy is also evenly distributed across the classes.

\begin{table}
  \caption{Attack accuracy and decrease in test accuracy on CIFAR10, 10000 devices, 200 attackers.}
  \label{table:table-stealth}
  \centering
  \begin{tabular}{lllll}
    \toprule
    % \multicolumn{2}{c}{Attack Accuracy (Dataset)}                   \\
    % \cmidrule(r){1-2}
    \textbf{Name}     & \textbf{Test acc decrease} & \textbf{Attack acc} \\
    \midrule
    Trimmed Mean & 4.78 & 100    \\
    Bulyan     & 7.35 & 92.6    \\
    Clipping     & 7.1 & 100\\
    \textbf{SparseFed (Ours)} & 6.61 & \textbf{25.6} \\
    \bottomrule
  \end{tabular}
\end{table}

Throughout the Appendix we show the tradeoff between benign accuracy and attack accuracy/OIF in tables and pareto curves in graphs, and leave the task of evaluating risk to practitioners.
We note that for the semantic backdoor task, the attack is not stealthy by definition.
\subsection{Range proofs for \algoname{}}\label{appendix:rangeproofs}
We do not go in depth on a proposed implementation of range proofs in a federated learning system for three main reasons.

First, prior work on defenses do not make any claims about the computational or communication efficiency of their proposed robust aggregation mechanisms, including the methods that we compare to in this work (Bulyan, Krum, etc.) This includes the works which initially proposed L2 norm clipping as a defense (Sun et. al. 2019). Given this, we did not feel that there is a precedent for defense papers which utilize L2 norm clipping and its variants to propose an efficient range proof that is compatible with existing systems, as this would fall more in the realm of an applied-cryptography/systems-security paper.

Second, through our industry experience, we know that not all existing deployments make use of secure aggregation due to its costly overhead and inefficiency at scaling up to larger numbers of clients. Because this is the case, a federated learning system which does not use secure aggregation can implement L2 norm clipping at the server very efficiently.

Third, to the best of our knowledge, all existing defenses against model poisoning attacks all need some degree of verification of client’s gradient updates whether it is L2 norm clipping or checking the sign of the gradient. SparseFed, unlike schemes which require consensus such as Bulyan or sign aggregation, does not require any additional secure computation beyond L2 norm clipping because there is no need to establish consensus between clients. In this regard, it is the most suited for deployment in a setting which requires secure aggregation assuming that a secure multiparty computation for L2 norm clipping has already been deployed.

Despite the above qualifications, we will now address the issue of how to implement range proofs for L2 norm clipping efficiently, using an informal description of how such a range proof can be achieved. While we do not provide details here, we believe that the method presented can lead to an actual proof in future work. The parties in a federated learning system are one server and one or more clients. The server will play the role of the verifier and the clients will be provers. Because our proposed protocol does not require any coordination between clients, without loss of generality we can simplify the system to one prover and one verifier. In the first step of the protocol, the prover generates a commitment to their update vector over the floating point domain. Next, the prover computes the sum of squares via a zkSNARK circuit (zero knowledge succinct non interactive argument of knowledge). Assuming that a custom SNARK is constructed for this application and the prover is using a standard multi-CPU chip found in the latest smartphones, the proving time would be less than thirty seconds (citation 1). This is minimal compared to the existing overhead in secure aggregation, which can take many minutes when accounting for multiple rounds of dropped users. If we want to be very conservative about how much information is leaked, we can treat the sum of squares as a secret committed value and use a bulletproof to ensure that it falls within the range of (0, L**2) where L is the L2 norm clipping constraint. Bulletproofs are fairly small and scale logarithmically in the number of commitments; we can validate all 100 L2 norms in one bulletproof for just 1MB in space, and all of this can be verified in ~2ms by the verifier’s hardware. If we can accept leaking the sum of squares, then we can just make it public and have the verifier check it outside the circuit. In either case, only provers who pass the verification will have their update vectors aggregated. This protocol sketch can be implemented without significantly increasing either the communication complexity (which is already quite large given that we have to at minimum upload gradients of deep networks) or the computation complexity (again, quite large because the device already has to compute gradients on local data).

\subsection{Tuning Attack Parameters}
\label{appendix:attacktuning}

\noindent \textbf{CIFAR attack parameters:}
We consider various numbers of attackers: $[100,200,400,1000]$ but most experiments are conducted with $100-200$ attackers which corresponds to having $1-2$ attackers present in every round.
We consider this to be in line with a real world threat model.
Typical federated learning training cycles take place over the course of a few days, and in order to use data from as many agents as possible, each round must draw data from many agents.
Agents are called on to participate when they fulfill a number of criteria, and an attacker can forge these criteria in order to control when they are selected.
Therefore, it should be straightforward for a small number of attackers to ensure that they are selected in every round.
All auxiliary datapoints are drawn from the CIFAR validation set.
Each point is randomly given a label from one of the 9 classes which it does not belong to.
There are a number of unique attack hyperparameters which we search over.
For the boosting factor, we search over [1,4,6,8,10,20] and find that a boosting factor of 20 works well for our experiments to ensure that PGD projects the update onto the perimeter of the $\ell_2$ constraint.
However, tuning the boosting factor does not make an impact whenever the $\ell_2$ defense is in place with a sufficiently small clipping threshold.
We tune the attacker's local batch size when they are doing PGD.
We use values of $[N/10,2N/10,4N/10,8N/10]$ where $N$ is the size of the auxiliary set.
We tune the number of epochs when the attacker is using the PGD attack.
We use values of $[1,3,5,7,9,11]$.
% We tune the local learning rate decay.
% We use values of $[0.2, 0.9, 1.0]$.

\subsubsection{Hyperparameter Tuning in Attacks}

% The two advanced attacks considered in this work are the distributed poisoning attack, which makes use of colluding attackers, and the projected gradient descent (PGD) attack which makes use of a white-box attacker.
The hyperparameters we consider are the attacker's \textbf{local batch size}, and the \textbf{number of local epochs for PGD}.

In Fig. \ref{fig:newdistpoison} we consider the impact of changing the attacker batch size across two different auxiliary set sizes: $500$ and $5000$, against the $\ell_2$ defense with parameter 5.
We find that varying the attacker batch size for the smaller auxiliary set size reveals a smooth pareto frontier which enables the attacker to double its attack efficacy against the $\ell_2$ for a moderate stealth budget when compared to the baseline attack.
Increasing the attacker batch size up to a certain point increases the efficacy of the attack at the expense of stealth; further increasing the attacker batch size does not continue moving along the pareto frontier.
This is because, as shown in our initial validation of the $\ell_2$ defense, attempting to backdoor the entire auxiliary set at every iteration for the smaller auxiliary set results in a very small OIF.

In Fig. \ref{fig:pgdnew} we tune the number of PGD epochs against the $\ell_2$ defense with parameter 5 at two different auxiliary set sizes, $500$ and $5000$.
Performing a larger number of gradient descent iterations over the auxiliary set overfits the gradient significantly, which enables the attacker to insert a backdoor with higher OIF at the expense of a considerable degree of stealth.
% Comparing the PGD attack and the distributed poisoning attack, we find that the distributed poisoning attack has a much better tradeoff for the attacker.

\begin{figure}[t]
\centering
\includegraphics[width=210px]{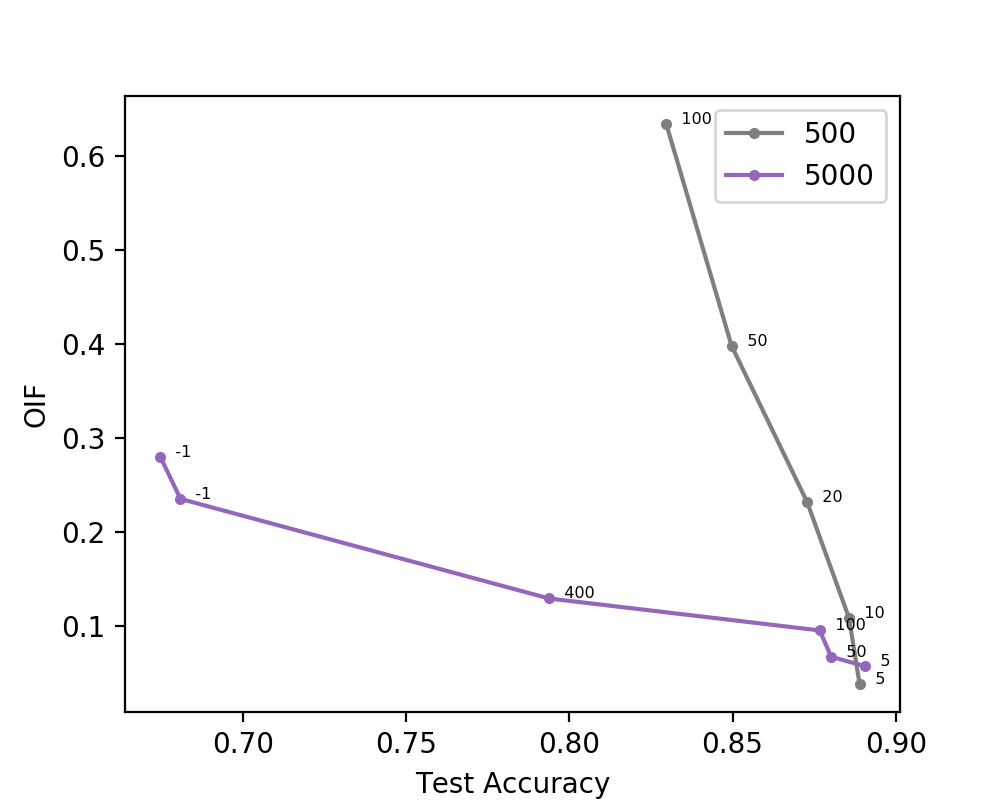}
\caption{
Pareto frontier of the attack when varying the batch size against the $\ell_2$ defense with a parameter of 5, using auxiliary set sizes of 500 and 5000.
While tuning the batch size does not achieve an OIF of $1$, it does improve the pareto frontier for the attacker. 
We find that varying the attacker batch size moves along the OIF-stealth tradeoff; larger backdoors correspond to better OIF, at the expense of stealth.
}
\label{fig:newdistpoison}
\end{figure}

\begin{figure}[t]
\centering
\includegraphics[width=210px]{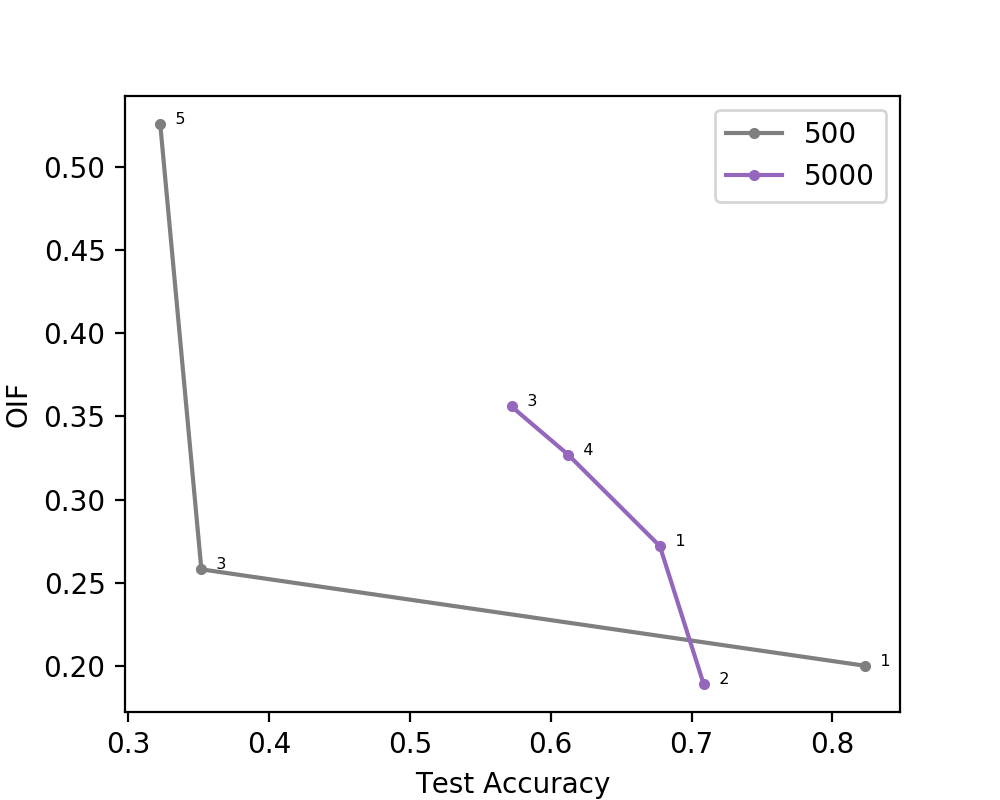}
\caption{
Pareto frontier of the PGD attack against the $\ell_2$ defense with a parameter of 5, using auxiliary set sizes of 500 and 5000. Increasing the number of epochs improves the OIF at the expense of stealth.}
\label{fig:pgdnew}
\end{figure}

\subsubsection{Additional Results}

In Fig. \ref{fig:auxbaseline} we vary the size of the auxiliary set to observe how successful a more "ambitious" attacker can be.
Generally, increasing the auxiliary set size enables the baseline attack to achieve a higher OIF at the expense of considerable stealth.
These results are summarized in the main body in Table 1.
% \ref{table:table-baseline-aux-effective}.

In Fig. \ref{fig:femnist}, we use the attack to insert a large number of backdoors against an undefended system on the FEMNIST dataset.
As mentioned in the main body, the OIF we obtain is notably $\approx 50 \times$ that of the attack benchmarked in prior work.
This is because we consider attackers that use a subset of the auxiliary set by minibatching, which enables us to use a much larger overall auxiliary set size in the attack.
These results are summarized in Table \ref{table:table-femnist-adaptive}.

In Fig. \ref{fig:cifar100baseline} we show the baseline attack against a system on CIFAR100 with $50000$ clients, each client possessing $1$ datapoint, $500$ workers and $100$ attackers.
When the system is undefended, the small number of attackers are able to insert an attack with OIF $1$.
However, enforcing the $\ell_2$ defense with parameter $5$ successfully mitigates this attack.
In the main body, we show results for the adaptive attack, where the attack reaches $100 \%$ accuracy against the $\ell_2$ defense.

\begin{figure}[t]
\centering
\includegraphics[width=210px]{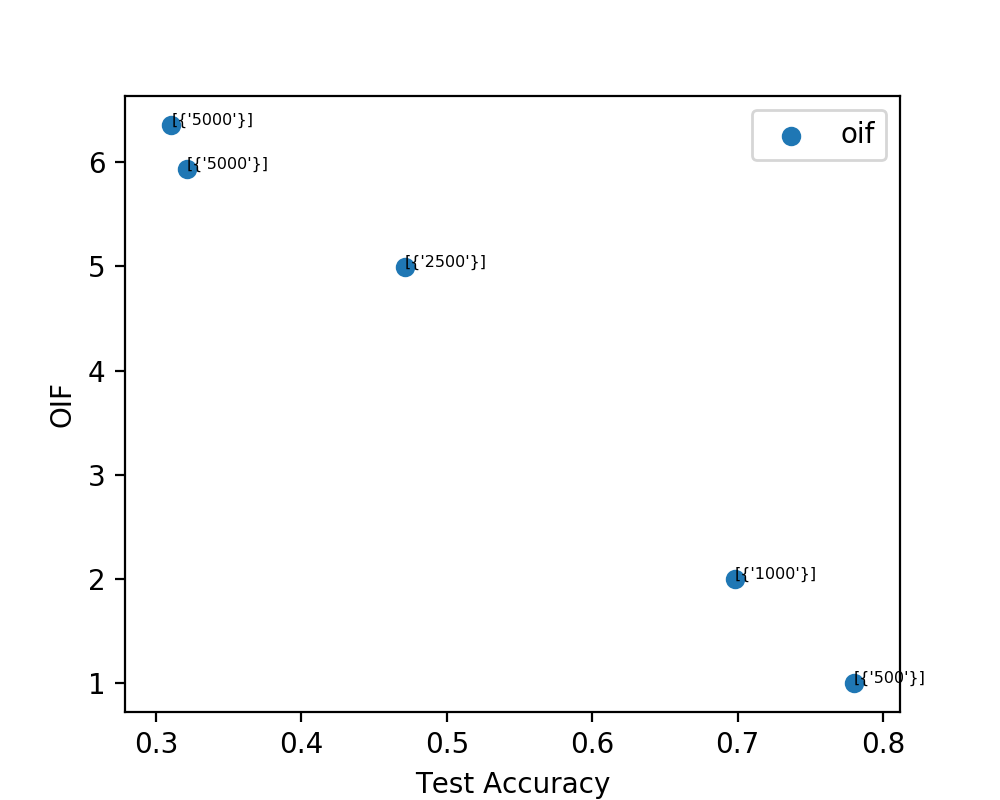}
\caption{
Pareto frontier of the baseline attack against the undefended system on CIFAR10 with 10000 clients and 100 workers.
Annotation is the size of the auxiliary set.
}
\label{fig:auxbaseline}
\end{figure}

\begin{figure}[t]
\centering
\includegraphics[width=210px]{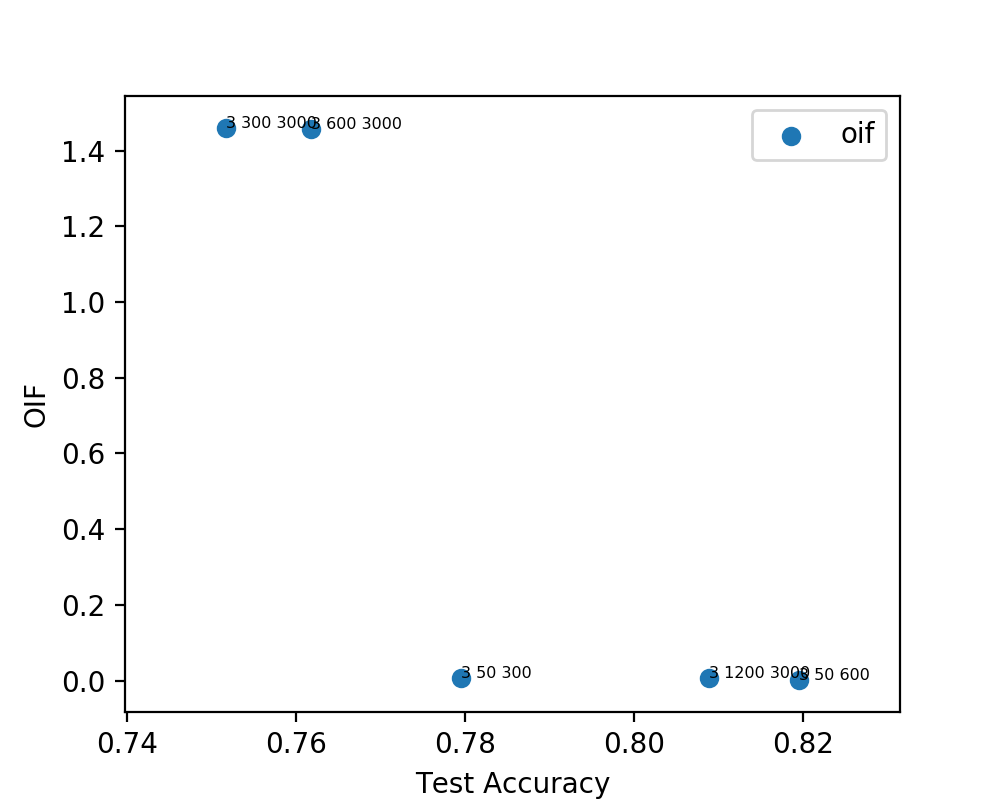}
\caption{
Pareto frontier of the attack against the undefended system on FEMNIST.
Annotation is the attacker batch size, and the size of the auxiliary set. Using a larger auxiliary set with an appropriately tuned batch size allows for much higher OIF.
}
\label{fig:femnist}
\end{figure}

\begin{figure}[t]
\centering
\includegraphics[width=210px]{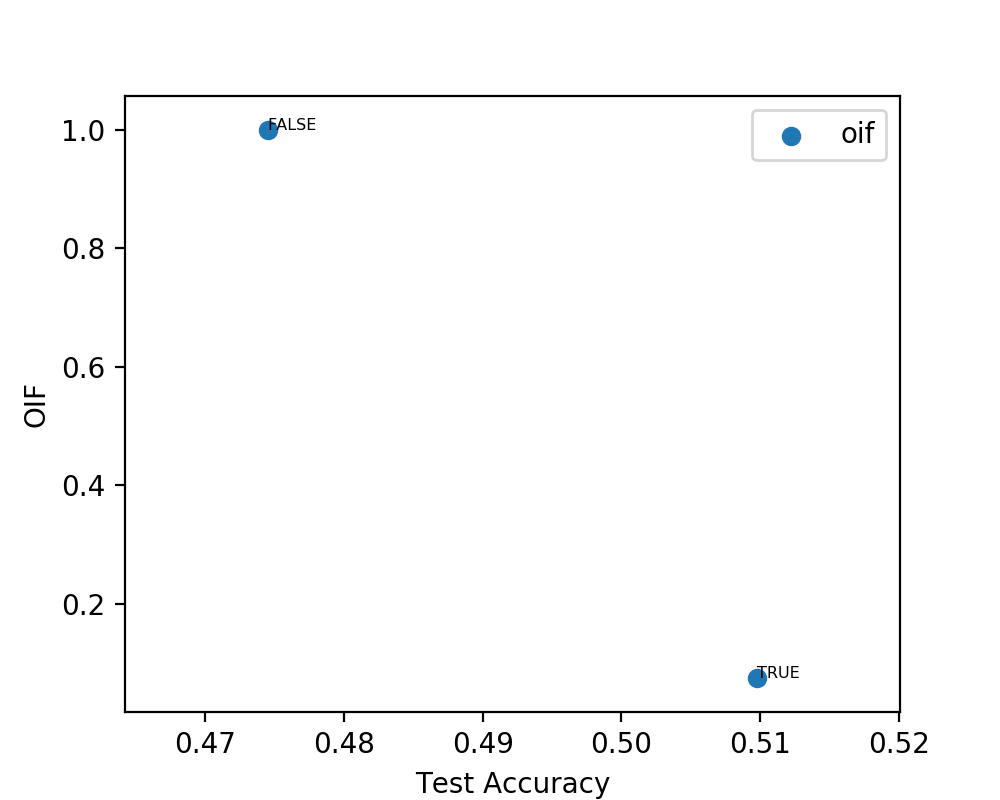}
\caption{
Baseline attack against CIFAR100 systems, with and without a DP-based $\ell_2$ defense in place.
}
\label{fig:cifar100baseline}
\end{figure}

In Figure \ref{fig:numattackers} we vary the number of attackers against various defenses.
We conclude that the defense which has the absolute highest robustness is: uncompressed \algoname{} with $k=d$, which is equivalent to uncompressed $\ell_2$ clipping without momentum.
However, the test accuracy of this approach is low ($44\%$). Overall, \algoname{} dominates the other defenses significantly, especially for a smaller number of attackers.

\begin{figure}[t]
\centering
\includegraphics[width=210px]{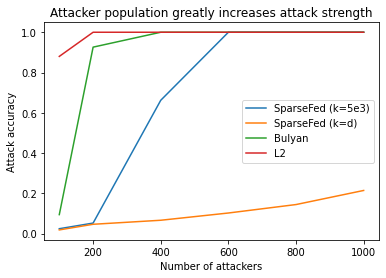}
\caption{
Attack against various defenses on CIFAR10 with varying number of attackers.
}
\label{fig:numattackers}
\end{figure}

In Table \ref{table:table-femnist-backdoor} we vary the nature of the semantic backdoor when attacking FEMNIST.
Instead of targeting the pair of digits $1$ and $7$, we target $4$ and $9$.
We find that both semantic backdoors perform similarly.

\begin{table}
\caption{Varying the semantic backdoor does not have a significant impact on the success of the attack against FEMNIST.}
\label{table:table-femnist-backdoor}
\begin{center}
\begin{small}
\begin{tabular*}{\columnwidth}{ccc}
\toprule
\textbf{Defense} & \textbf{Attack acc (1/7)}& \textbf{Attack acc (4/9)}\\
 \midrule
$\ell_2$ & 100 & 100 \\
\algoname{} & 1.95 & 6.72 \\
\bottomrule
\end{tabular*}
\end{small}
\end{center}
\end{table}
\subsection{\fedssgd{}}
\label{appendix:fetchsgd}

\subsubsection{The Case for Sparsification}
% \begin{table}
% \caption{Freezing the model near convergence at the appropriate layer can mitigate the effectiveness of the adaptive attack, as a heuristic defense on top of the norm clipping defense.
% Freezing all layers but the last layer gives the best performance.}
% \label{table:table-layerfreezing}

% \begin{center}
% \begin{small}
% \begin{tabular*}{\columnwidth}{ccc}
% \toprule
% \textbf{No. unfrozen layers} & \textbf{Test acc}& \textbf{OIF}\\
%  \midrule
% 2 & 0.82250 & 0.345 \\
% \textbf{1} & 0.82800 & 0.295 \\
% \bottomrule
% \end{tabular*}

% \end{small}
% \end{center}
% %\vspace*{-20pt}

% \end{table}
% In Table \ref{table:table-layerfreezing} we evaluate the heuristic layer freezing defense with the $\ell_2$ defense.
% From analyzing the convergence behavior of the benign model, we determine that the model has mostly converged after the first $80 \%$ of training.
% After this number of iterations has elapsed, we freeze all parameters of the global update up until the last layer.
% This mitigates even colluding attackers using the PGD attack.
\begin{figure}[t]
\centering
% \begin{minipage}{.4\linewidth}\centering

\includegraphics[width=210px]{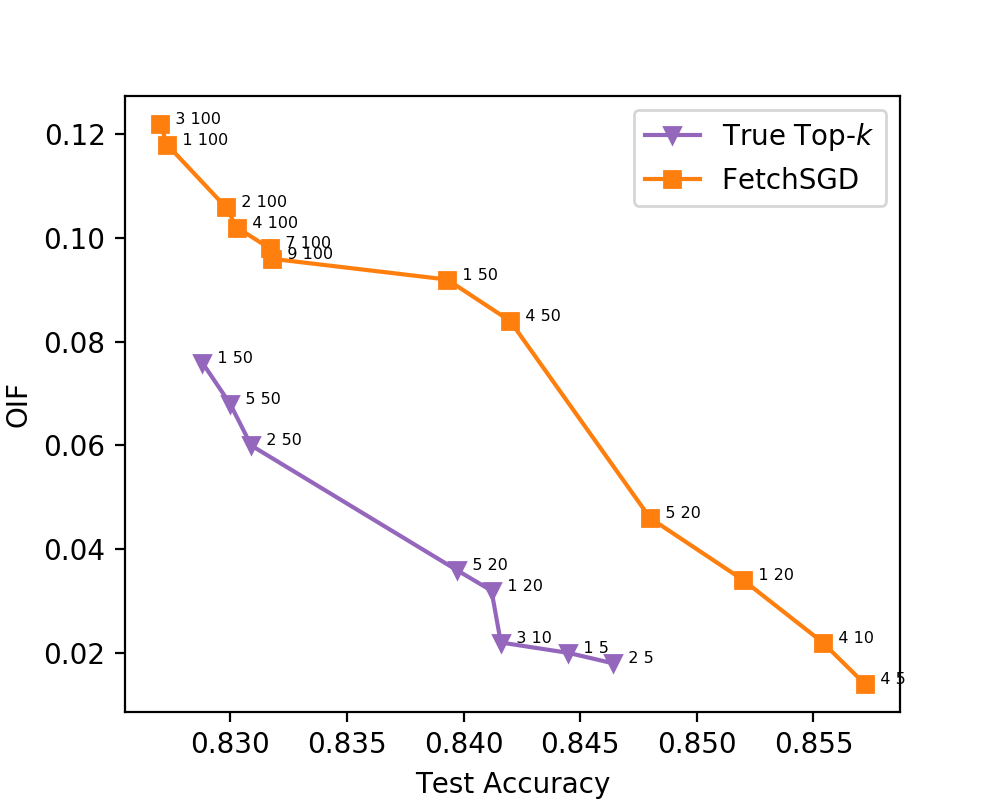}
%   %\vspace{.5cm} %ICML
%\vspace{-.5cm} % ICML
\caption{Pareto frontier of \algoname{} using $top_k$ and \fedssgd{} with $\ell_2$ clipping using parameter 5, against varying hyperparameters of the colluding PGD attack, with a fixed auxiliary dataset of size 500.
This is the best that the strongest available attack can perform against our defense, and we achieve a factor of $5-10 \times$ improvement over the $\ell_2$ defense.
}
%\vspace{-.5cm}
\label{fig:sparsevsadaptive}
        %%\vspace{-.5cm} %ICML
        % \end{minipage}%

\end{figure}

In Fig. \ref{fig:sparsevsadaptive} we evaluate our provable defense using two implementations of \algoname{}:  top-$k$ and FetchSGD sparsification.
As an implementation detail, here we use top-$k$ and the $\ell_2$ defense in the uncompressed setting, and FetchSGD is in the "uncompressed" setting where the overall communication cost is reduced by a factor of $10$.
In all experiments, we update only $k=5e4$ gradient parameters at every iteration.
We see that for a defended system with a moderate stealth threshold of $5\%$, the attack achieves $0.05$ OIF.
Thus our \algoname{} defense outperforms the $\ell_2$ defense by a factor of $10 \times$ (recall that the $\ell_2$ defense incurs an OIF of 0.5 under comparable constraints in Figure~\ref{fig:adaptivevsncd}).
Both implementations mitigate the attack, and using FetchSGD for robustness simultaneously achieves communication efficiency and enables us to operate in the uncompressed setting where we gain further robustness.
% The combination of FetchSGD, which approximates top-k in the $\ell_2$ norm, and $\ell_2$ norm clipping, can be understood as an election wherein all agents cast a bounded number of votes.

\section{Limitations and societal impact}\label{appendix:limitations}
\noindent \textbf{Limitations:} 
The empirical limitation of our work is that we are forced to make imperfect simulations of cross-device federated settings because we do not have access to real federated datasets at the scale of tens of thousands of devices.
For CIFAR10, CIFAR100, and FMNIST, lacking any natural non-iid partitioning, our simulation strategy is to simulate each device only drawing samples from the distribution of one class than multiple classes, but this may not necessarily be true in the real world. 
We encourage the federated learning community to contribute real-world and large-scale datasets to overcome such limitations in the future. 

\noindent \textbf{Security considerations:} 
We recognize that our analysis of existing Byzantine resilient defenses reveals that colluding attackers can successfully attack systems which may use these defenses today.
To mitigate these attacks, we urge stakeholders in these deployed systems to inspect their vulnerabilities using the same powerful attacker we use in our work.
\vspace*{-2pt}
The field of federated learning has seen a great deal of research interest lately.
Federated learning systems today utilize data from millions of users and serve millions more, so adversarial robustness is of paramount importance.
Prior work in the field of targeted model poisoning attacks has examined the impact that attacks have in the cross-silo setting, and mostly concluded that 
In this work, we complement this body of work by demonstrating the outsized impact of model poisoning attacks on systems at scale and showing that existing defenses can be broken by colluding attackers.
We also introduce \algoname{}, and prove practical robustness guarantees for our novel defense.
We compare \algoname{} to existing defenses, and confirm that it outperforms these against our strongest available attacks empirically at large scales.
Although future work may introduce attacks which are stronger than we consider, we emphasize that \algoname{} will maintain provable robustness against any attack.
We leave investigation of the tradeoffs between other proposed attacks and defenses to future work.
% In the future, further study of the connection between differential privacy and the robustness guarantees we prove in this work can improve both our 
% \input{sections/proof}
% \input{sections/appendix_furtherexperiments}
% \input{sections/cross_entropy_loss}
% \input{sections/old_theory}
{\small
\bibliographystyle{plainnat}
\bibliography{fed_learn}}
\end{document}